\documentclass[mathpazo]{cicp}

\usepackage[margin=2.5cm]{geometry}
\usepackage{setspace}
\usepackage{lmodern}
\usepackage{amsmath,amssymb}
\usepackage{amsfonts,amsthm}
\usepackage{lineno}
\usepackage{graphicx}
\usepackage{bm, bbm}
\usepackage{algorithm}
\usepackage{algpseudocode}

\usepackage{physics}
\usepackage{color}
\usepackage{pxfonts}
\usepackage[footnotesize,bf]{caption}
\usepackage{subcaption}
\usepackage{mathtools}
\usepackage{hhline}
\usepackage[T1]{fontenc}

\usepackage{caption}
\usepackage{graphicx}
\usepackage{float}
\usepackage{subcaption}

\usepackage{mathrsfs}
\usepackage{placeins}

\numberwithin{equation}{section}
\begin{document}
\title{Frequency-adaptive tensor neural networks for high-dimensional multi-scale problems}


\author[Huang J Z et.~al.]{
      Jizu Huang\affil{1,2}, and Yue Qiu\affil{3}, Rukang You\affil{1,2}\comma\corrauth}

\address{\affilnum{1}\ SKLMS, Academy of Mathematics and Systems Science, 
        Chinese Academy of Sciences, 
        Beijing, 100190, PR China.\\
        \affilnum{2}\ School of Mathematical Sciences, 
        University of Chinese Academy of Sciences, 
        Beijing 100190, PR China.\\
        \affilnum{3}\ College of Mathematics and Statistics, 
        Chongqing University.}

\emails{{\tt yourukang@lsec.cc.ac.cn} (R.~You)}

\begin{abstract}
  Tensor neural networks (TNNs) have demonstrated 
  their superiority in solving high-dimensional 
  problems. However, similar to conventional neural networks, 
  TNNs are also influenced by the Frequency Principle,
  which limits their ability to accurately capture
  high-frequency features of the solution.
  In this work, we analyze the training dynamics of TNNs by Fourier analysis and enhance their expressivity for high-dimensional multi-scale problems by incorporating random Fourier features.
  Leveraging the inherent tensor structure of TNNs, we further propose a novel approach to extract frequency features of high-dimensional functions by performing the Discrete Fourier Transform to one-dimensional component functions.
  This strategy effectively mitigates the curse of dimensionality.  
  Building on this idea, we propose a frequency-adaptive 
  TNN algorithm, which significantly improves the 
  ability of TNNs in 
  solving complex multi-scale problems. 
  Extensive numerical experiments are performed to validate the effectiveness 
  and robustness of the proposed frequency-adaptive TNN algorithm.
\end{abstract}


\maketitle

\section{Introduction}

Building upon their groundbreaking achievements in computer vision 
\cite{krizhevsky2012imagenet}, 
speech recognition \cite{hinton2012deep}, 
and natural language processing 
\cite{vaswani2017attention, devlin2018bert, brown2020language}, 
deep neural networks (DNNs) have emerged as a promising paradigm for scientific computing, particularly in
solving partial differential equations (PDEs) 
\cite{han2017deep, yu2018deep, zang2020weak, tran2019dnn, han2018solving, han2017deep2, he2018relu, liao2019deep, raissi2019physics, strofer2019data, wang2020mesh}. 
Despite this promise, DNN-based methods still face significant theoretical and computational challenges 
compared to traditional numerical methods such as the finite element and finite volume methods. 
In particular, they often struggle to achieve high-fidelity solutions for
multi-scale problems, especially in high-dimensional settings. Although DNNs
have demonstrated strong approximation capabilities for high-dimensional
functions \cite{han2017deep, zang2020weak, han2018solving, zeng2022adaptive}, their practical
performance is frequently limited by high training costs and optimization
difficulties. 
A further critical challenge arises from the well-known  Frequency Principle (F-Principle), or spectral bias
\cite{rahaman2019spectral, xu2018understanding, xu2019frequency, zhang2019explicitizing, xu2019training}, 
which states that DNNs tend to learn low-frequency 
components of a target function earlier during training. 
This inherent bias limits the ability of DNNs to efficiently 
capture high-frequency features, posing significant obstacles to solving high-dimensional 
multi-scale problems characterized by a broad frequency spectrum. 
These challenges have motivated the development of novel network
architectures that aim to improve computational scalability, training efficiency,
and solution accuracy. Therefore, an effective neural PDE solver for 
high-dimensional multi-scale problems must address two intertwined 
difficulties: the efficient representation and computation of 
high-dimensional solutions, and the spectral bias of neural networks.

To address the first difficulty, tensor decomposition methods provide a natural
framework for improving the representation and computational efficiency of
high-dimensional functions. Classical tensor formats, such as
CANDECOMP/PARAFAC (CP) decomposition \cite{hitchcock1927expression}, Tucker
decomposition \cite{tucker1966some}, and Tensor Train (TT) decomposition
\cite{oseledets2011tensor}, exploit low-rank structures in high-dimensional
functions and have been widely used in high-dimensional approximation and PDE
computation \cite{dahmen2016tensor, schwab2011sparse}. These methods are
motivated by the observation that many high-dimensional PDE solutions possess
hidden low-dimensional structures or can be well approximated on low-complexity
manifolds \cite{cohen2011analytic, bachmayr2017kolmogorov}, thereby improving
computational scalability while maintaining approximation accuracy. 
Recent studies have further connected tensor decomposition methods with 
DNNs through tensorized neural networks
\cite{novikov2015tensorizing}, multiple-input operators \cite{jin2022mionet},
and physics-informed Tensor Neural Networks (TNNs)
\cite{vemuri2025functional, wang2022tensor} inspired by Physics-Informed Neural
Networks (PINNs) \cite{raissi2019physics}. These advances enhance the efficiency
and scalability of neural PDE solvers for high-dimensional problems.
Nevertheless, most existing tensor-based neural methods primarily focus on
high-dimensional representation and computation, while their ability to capture
multi-scale or high-frequency structures remains insufficiently understood.

To overcome the spectral bias of neural networks in learning high-frequency
components, several frequency-enhanced strategies have been developed for
DNN-based methods. Multi-scale Deep Neural Networks (MscaleDNNs)
\cite{liu2020multi, zhang2023correction} employ radial down-scaling mappings
in the frequency domain, which transform high-frequency components into
lower-frequency ones that are easier for standard networks to learn. Since these
radial mappings are independent of the spatial dimension, MscaleDNNs are
particularly attractive for high-dimensional multi-scale problems. Another
effective strategy is to use Fourier feature mappings \cite{tancik2020fourier}
or random Fourier feature mappings \cite{wang2021eigenvector}, which enrich the
input representation and enhance the network's ability to capture
high-frequency components. Despite their effectiveness, these approaches
usually rely on pre-defined frequency parameters, and their performance can be
sensitive to the choice of these parameters. To reduce such parameter
sensitivity, Huang et al. \cite{huang2024frequency} proposed frequency-adaptive MscaleDNNs 
based on a rigorous analysis of the approximation error of DNNs. Leveraging the Discrete Fourier
Transform (DFT), this adaptive approach dynamically adjusts the frequency features during training, 
thereby improving both robustness and accuracy for multi-scale problems. 
However, the integration of such adaptive frequency mechanisms into TNN
frameworks for high-dimensional multi-scale problems remains largely unexplored.

Motivated by the above observations, this work develops frequency-enhanced TNN
frameworks for solving high-dimensional multi-scale PDEs. Building upon PINNs and 
functional tensor decompositions \cite{vemuri2025functional}, we 
introduce two architectures: Canonical Polyadic decomposition-based PINNs (CP-PINNs) and Tensor Train decomposition-based PINNs (TT-PINNs). 
Inspired by the analysis in \cite{xu2018understanding}, we prove that TNNs, similar to standard DNNs, are also affected by the F-Principle. 
This inherent spectral bias limits the ability of TNNs to efficiently capture
high-frequency features, as further confirmed by our numerical experiments. To enhance 
the performance of TNNs on high-dimensional multi-scale problems, we first incorporate random Fourier 
feature mappings into the TNN framework. When the parameters of the Fourier mappings are properly chosen, this combination significantly improves the accuracy of TNNs. 
However, the effectiveness of this approach still depends heavily on the choice of pre-defined frequency parameters. 
To alleviate this sensitivity, we propose a novel frequency-adaptive TNN algorithm, inspired by the work of \cite{huang2024frequency}. 
In this approach, frequency features are adaptively identified by applying the DFT to the component functions of TNNs. Since these 
component functions are one-dimensional, the computational cost of performing the DFT is significantly lower than that of applying it to 
the high-dimensional output of the full network, thereby substantially reducing the computational burden of frequency
identification in high dimensions. Although the resulting frequency set can be oversampled and may contain noisy
components, numerical experiments demonstrate that the proposed frequency-adaptive TNN algorithm 
substantially improves the accuracy of TNNs, reducing the solution error by up to two or three orders of magnitude.

The remainder of the paper is organized as follows.
Section \ref{tnn_lr_fp} introduces two TNN architectures and investigates their F-Principle behavior, based on the 
Kolmogorov $r$-width structure of high-dimensional PDE solutions.
Section \ref{ff_for_tnn} integrates random Fourier features into the TNN framework to address multi-scale problems.
Section \ref{fa} presents the frequency-adaptive TNN algorithm.
Section \ref{ne} provides a comprehensive numerical evaluation using a set of high-dimensional benchmark problems.
Finally, section \ref{conclusion} summarizes the key findings and contributions of this work.

\section{Tensor Neural Networks and their Frequency Principle}
\label{tnn_lr_fp}

High-dimensional PDEs, 
especially those arising in multi-scale physical systems, often admit 
solutions that lie close to low-rank manifolds within suitable
function spaces. This intrinsic low-rank structure motivates 
the development of neural network architectures capable of 
efficiently capturing and exploiting such properties. In 
section \ref{low_rank}, we further illustrate this phenomenon 
using the concept of Kolmogorov $r$-width, 
demonstrating that solutions to high-dimensional PDEs can often 
be accurately approximated by low-rank representations. Building 
on this insight, 
section \ref{ftd} introduces TNNs, which extend the PINN framework 
by incorporating classical tensor low-rank decompositions. 
We present two variants of TNNs: CP-PINNs and TT-PINNs. 
These architectures are particularly well-suited for capturing 
low-rank structures in high-dimensional settings and offer 
enhanced representational efficiency over conventional
neural networks.
In section~\ref{fp_tnn}, we investigate the spectral bias, also 
known as the F-Principle \cite{rahaman2019spectral, xu2019frequency}, 
of TNNs and provide a theoretical explanation of this behavior.

\subsection{Low-rank Representation of High-dimensional Functions}
\label{low_rank}

The low-rank structure of solutions to high-dimensional PDEs in 
infinite-dimensional function spaces has been discussed 
in \cite{cohen2011analytic, bachmayr2017kolmogorov}.
For an elliptic equation $\mathscr{L}u=f$, 
we define the associated solution manifold as:
\begin{equation*}
  \mathfrak{M}:=\{u_f: \mathscr{L}u_f = f, \  f\in \mathfrak{F}\},
\end{equation*}
where $\mathscr{{L}}$ is the differential operator and 
$\mathfrak{F}$ denotes the set of admissible source terms. 
To reduce the complexity of representing high-dimensional functions, 
we approximate the manifold $\mathfrak{M} \subset V$ 
by a low-rank nonlinear manifold in a tensor Banach space $V$, 
as discussed in \cite{falco2019dirac}. 
The space $V$ is typically endowed with a tensor product structure:
\begin{equation*}
  V:=V_1 \otimes V_2 \otimes \cdots \otimes V_d,
\end{equation*}
where each $V_j$ is a normed space corresponding to the 
$j$-th variable. As proved in \cite{falco2019dirac}, the set 
of tensors with a fixed Tucker rank forms an immersed 
sub-manifold of $V$, and moreover, there exists a best approximation 
in the Tucker format (see Corollary 4.18 therein). 

Inspired by this result, we assume that the solution $u_f$ can be approximated 
using the following CP decomposition:
\begin{equation*}
  u_f(x_1,x_2, \dots, x_d) \approx u_f': = \sum_{\alpha=1}^r u_{1,\,\alpha}(x_1)u_{2,\,\alpha}(x_2)\cdots u_{d,\,\alpha}(x_d),
\end{equation*}
where each function $u_{j,\,\alpha} \in V_j, \, \alpha=1,\dots, \, r$. Define the rank-$r$ manifold as:
\begin{equation*}
  {\mathfrak{M}}_r = \left\{\sum_{\alpha=1}^r u_{1,\,\alpha}(x_1)u_{2,\,\alpha}(x_2)\cdots u_{d,\,\alpha}(x_d)\,\Big| \ u_{j,\,\alpha} \in V_j, \ u_{j,\,\alpha}\neq0,\,\ \alpha=1,\dots,\, r, \ j=1,\dots,\, d\right\}.
\end{equation*}
The Kolmogorov $r$-width in the space $V$ is 
defined by
\begin{equation*}
  d_r(\mathfrak{M}, V):= \inf_{{\mathfrak M}_r \subset V} ~~\sup_{u_f\in \mathfrak{M}}~~ \inf_{u_f'\in {\mathfrak M}_r} \|u_f - u_f'\|_{V},
\end{equation*}
which is commonly used to quantify the approximation capacity of rank-$r$ subsets 
and assess the quality of low-rank approximations to $\mathfrak{M}$.
Here: 
\begin{itemize}
  \item $d_r(\mathfrak{M}, V)$ measures the worst-case error of best rank-$r$ 
  approximation to the elements of $\mathfrak{M}$ 
  using functions from $\mathfrak{M}_r$;
  \item A rapid decay of $d_r$ (e.g., exponential decay) indicates that 
  $u_f\in\mathfrak{M}$ lies near a rank-$r$ manifold.
\end{itemize}

The Kolmogorov $r$-width serves as a rigorous measure of the compressibility 
of the solution manifold $\mathfrak M$. A small $r$-width 
implies that high-dimensional functions can be effectively approximated 
using low-rank structures, such as reduced 
basis methods \cite{patera2007reduced}, tensor 
decompositions \cite{khoromskij2011tensor}, or 
separable neural networks \cite{chen2024solving}, particularly in the context 
of parameterized and high-dimensional PDEs. Remarkably, this favorable 
low-rank behavior is often preserved even for nonlinear PDEs, provided the 
solution operator possesses sufficient regularity, such as compactness or Lipschitz continuity with respect to separable variations in the input data. 
This observation helps to explain the empirical success of the TNN-based methods 
proposed in this work, especially in their ability to capture the intrinsic 
structure of high-dimensional solution manifolds while significantly reducing 
computational complexity.

\subsection{Tensor Neural Networks}
\label{ftd}

An effective approach for approximating multivariate functions, 
particularly in high-dimensional settings, 
is the separation of variables. 
A classical method for implementing this idea is
tensor decomposition, which represents high-dimensional functions
as outer products of univariate functions, a technique commonly 
known as {\it functional tensor decomposition}. In this work, we 
focus on two widely used tensor decomposition formats, the CP
decomposition \cite{hitchcock1927expression} and the
TT decomposition \cite{oseledets2011tensor}, to construct TNNs. As shown 
in Figure~\ref{ftd_pinn}, each component function of these tensor decompositions 
is approximated by an individual neural network.
Moreover, we illustrate how these decompositions can be seamlessly integrated 
into the PINN framework \cite{vemuri2025functional}, 
enabling efficient
approximation of solutions to high-dimensional PDEs.

\begin{figure}[H]
  \centering
  \begin{subfigure}[b]{0.45\textwidth}
      \centering
      \includegraphics[width=\textwidth]{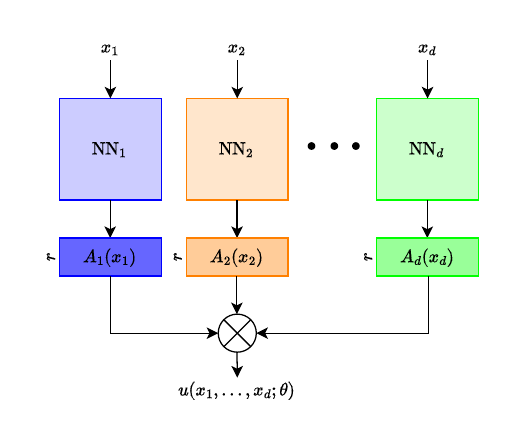}
      \caption{CP-PINNs}
      \label{tnn_cp}
  \end{subfigure}
  \hfill
  \begin{subfigure}[b]{0.45\textwidth}
      \centering
      \includegraphics[width=\textwidth]{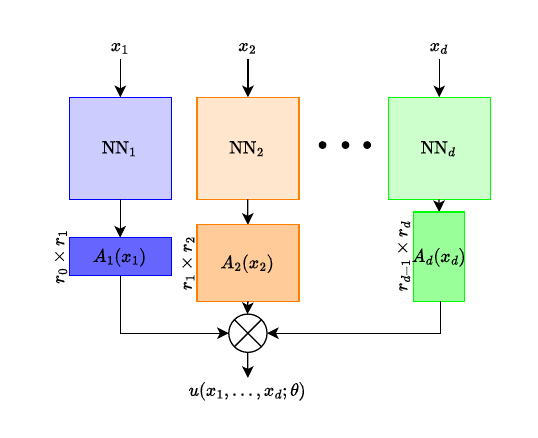}
      \caption{TT-PINNs}
      \label{tnn_tt}
  \end{subfigure}
  \caption{In TNNs, 
  each component function, corresponding to a single variable, is approximated
  by an individual neural network. 
  The resulting outputs are then combined according to  
  (a) the CP and (b) the TT decomposition.}
  \label{ftd_pinn}
\end{figure}

\textbf{CP decomposition} decomposes a tensor of order
$d$ into $d$ factor components with a specified rank $r$ \cite{hitchcock1927expression}. 
For a multivariate function $u(x_1,x_2,\dots,x_d)$, the functional CP decomposition 
takes the form 
\begin{equation}
  \label{cp_eq}
  u(x_1,x_2,\dots,x_d) \approx \left[\mathbf{A}_1(x_1)\odot \mathbf{A}_2(x_2) \dots \odot \mathbf{A}_d(x_d) \right] \cdot \bm{1}_r,
\end{equation}
where each factor component $\mathbf{A}_j (x_j):=(u_{j,\,1}(x_j),\,\ldots,\,u_{j,\,r}(x_j)) \in \mathbb{R}^r$ for $j=1,\dots,\ d$, $\bm{1}_r$ is an $r$-dimensional 
column vector with all entries equal to $1$,  
$\odot$ denotes the Hadamard (element-wise) product, 
and $\cdot$ represents the inner vector product.
This decomposition \eqref{cp_eq} can be
incorporated into the TNN-based PINN framework, 
as illustrated in Figure~\ref{tnn_cp}, resulting in the CP-PINN architecture. 
To further enhance the approximation ability of CP-PINNs, we can replace the 
fixed vector $\bm{1}_r$ with a
trainable parameter vector $\mathbf{W}\in\mathbb{R}^{d}$.

\textbf{TT decomposition} represents a high-dimensional 
tensor as a sequence of low-dimensional 
tensors (cores) connected in a chain-like structure, 
hence the name "train" \cite{oseledets2011tensor}.
For a multivariate function $u(x_1,x_2,\dots,x_d)$, the functional TT decomposition is expressed as: 
\begin{equation}
  \label{tt_eq}
  u(x_1,x_2,\dots,x_d) \approx \mathbf{A}_1(x_1)\times_1 \mathbf{A}_2(x_2) \times_1 \dots \times_1 \mathbf{A}_d(x_d),
\end{equation}
where $\times_1$ denotes the contracted product, and each core is given by
$\mathbf{A}_j(x_j):=(u_{j,\,\alpha_1,\,\alpha_2}(x_j))\in \mathbb{R}^{r_{j-1}\times r_j}$.
{The multi-TT ranks are $r_0, \ r_1, \dots, r_d$ with $r_0=r_d=1$. 
The contracted product $\times_1$ refers to the contraction of the last index of the first core tensor with the first index of the second core tensor.} 
In contrast to the CP decomposition, the TT decomposition
establishes explicit connections between the tensors associated with adjacent dimensions, improving numerical stability and expressiveness. 
This decomposition \eqref{tt_eq} can also be effectively integrated into 
the TNN-based PINN framework, as illustrated in Figure~\ref{tnn_tt}, 
resulting in the TT-PINN architecture.

Tensor product decomposition has been widely used to construct low-rank approximations 
of operators and functions \cite{jin2022mionet, wang2022tensor, wang2024solving, vemuri2025functional}. 
In particular, \cite{vemuri2025functional} introduces a class of architectures known as TNNs, 
which take advantage of the expressive efficiency of low-rank tensor formats to approximate 
high-dimensional functions with reduced parameter complexity. 
This makes TNNs especially well-suited for solving high-dimensional PDEs that exhibit low-rank 
structures.
Notably, TNNs have demonstrated significant advantages in capturing key features of high-dimensional 
solutions and in mitigating the curse of dimensionality, which often arises in high-dimensional 
integration and function approximation tasks \cite{wang2022tensor, wang2024solving}.
However, their performance in addressing high-dimensional PDEs with significant high-frequency 
components remains largely unexplored, an issue that will be thoroughly investigated in this work.

\subsection{Frequency Principle of Tensor Neural Networks}
\label{fp_tnn}

The F-Principle 
\cite{xu2018understanding, xu2019frequency, zhang2019explicitizing, xu2019training}, 
also known as spectral bias \cite{rahaman2019spectral},
has been extensively studied in classical neural networks for solving PDEs with high-frequency components. 
Since TNNs share many
architectural similarities with standard feedforward networks, they also exhibit a similar tendency to prioritize low-frequency components when approximating multi-scale solutions. To investigate this phenomenon, we introduce a toy model that illustrates how this frequency bias manifests in TNNs and why it hampers their ability to accurately capture high-frequency features. In this subsection, a theoretical explanation of this behavior is provided.

\begin{figure}[H]
  \centering
  \begin{subfigure}[b]{0.35\textwidth}
    \includegraphics[width=\textwidth]{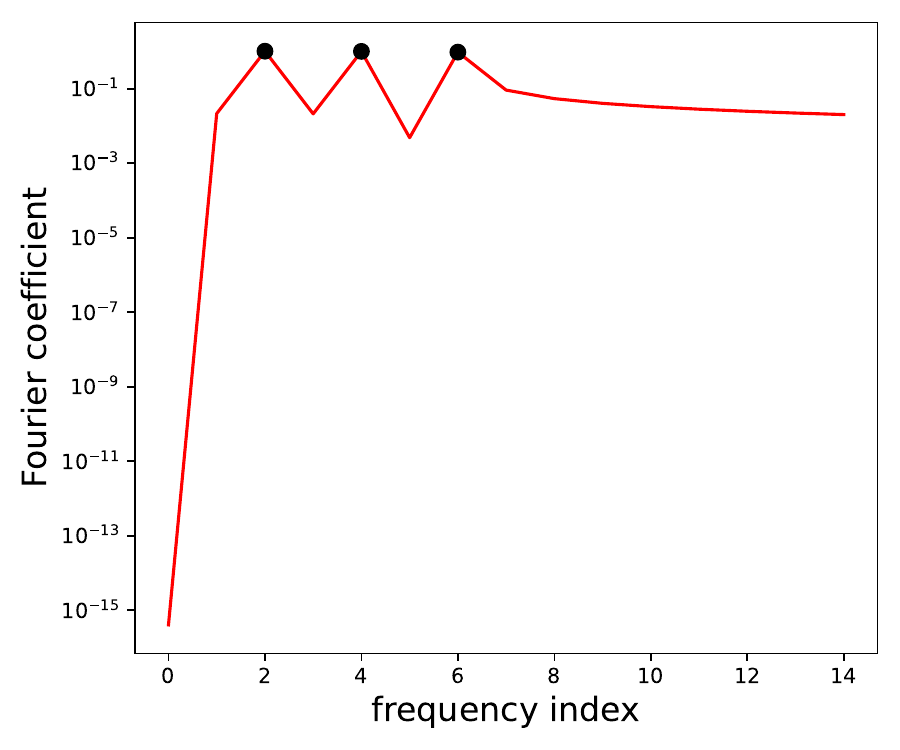}
    \caption{Frequency function $|\hat{f}(k_x, 0)|$.}
    \label{fig:spectral_bias_tnn_a}
  \end{subfigure}
   \hspace{1cm}
  \begin{subfigure}[b]{0.35\textwidth}
    \includegraphics[width=\textwidth]{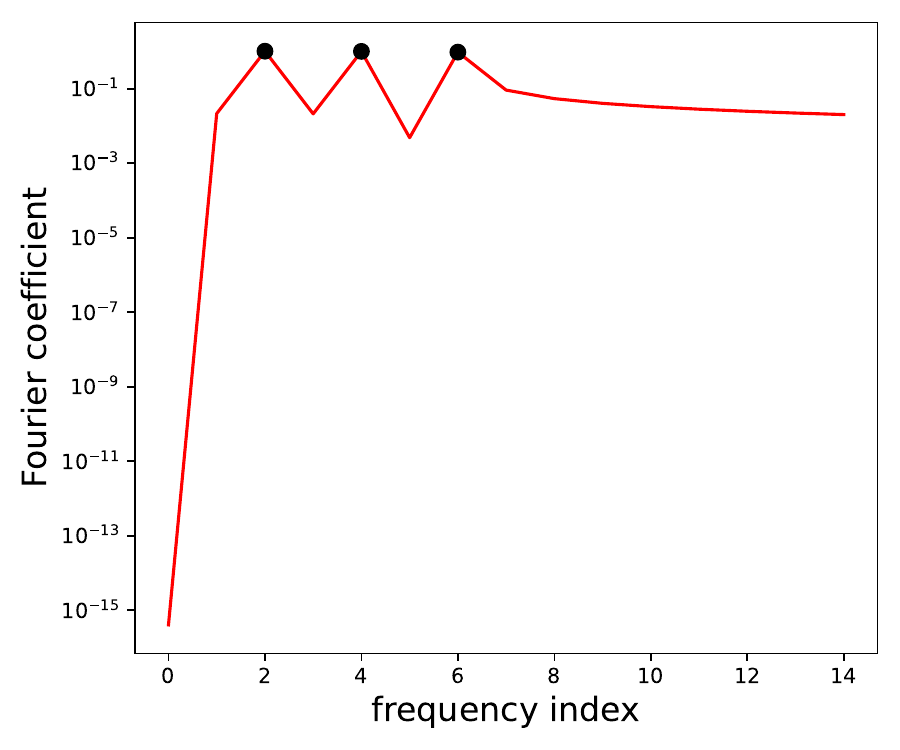}
    \caption{Frequency function $|\hat{f}(0, k_y)|$.}
    \label{fig:spectral_bias_tnn_b}
  \end{subfigure}
  \\
  \hspace{0.65cm}  
  \begin{subfigure}[b]{0.35\textwidth}
      \includegraphics[width=\textwidth]{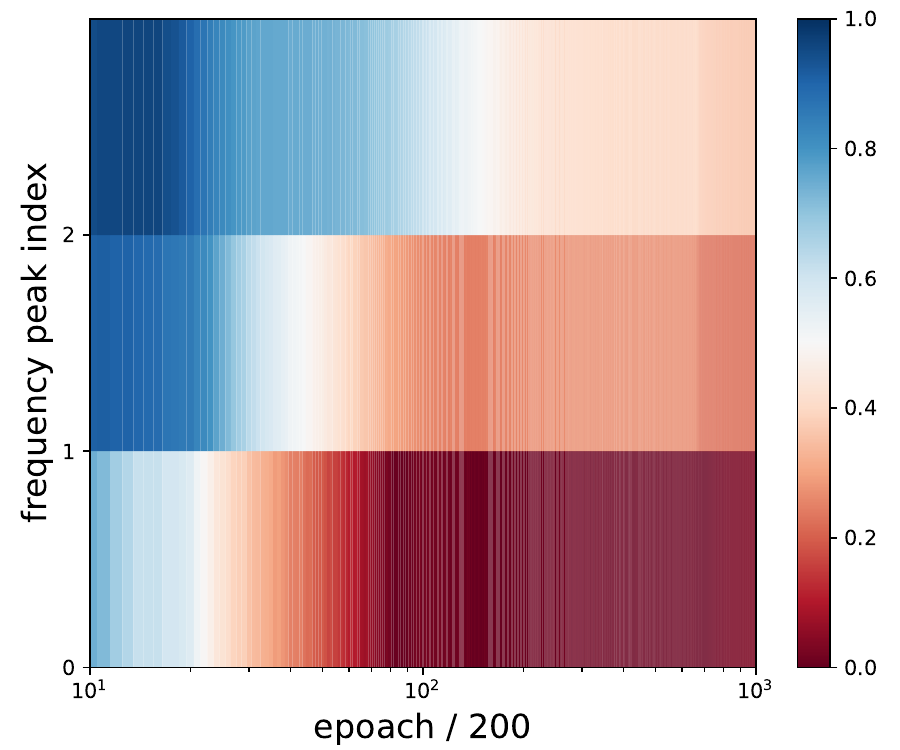}
      \caption{Convergence of TNNs for three frequencies in the $x$ direction.}
      \label{fig:spectral_bias_tnn_c}
  \end{subfigure}
     \hspace{1cm}
  \begin{subfigure}[b]{0.35\textwidth}
      \includegraphics[width=\textwidth]{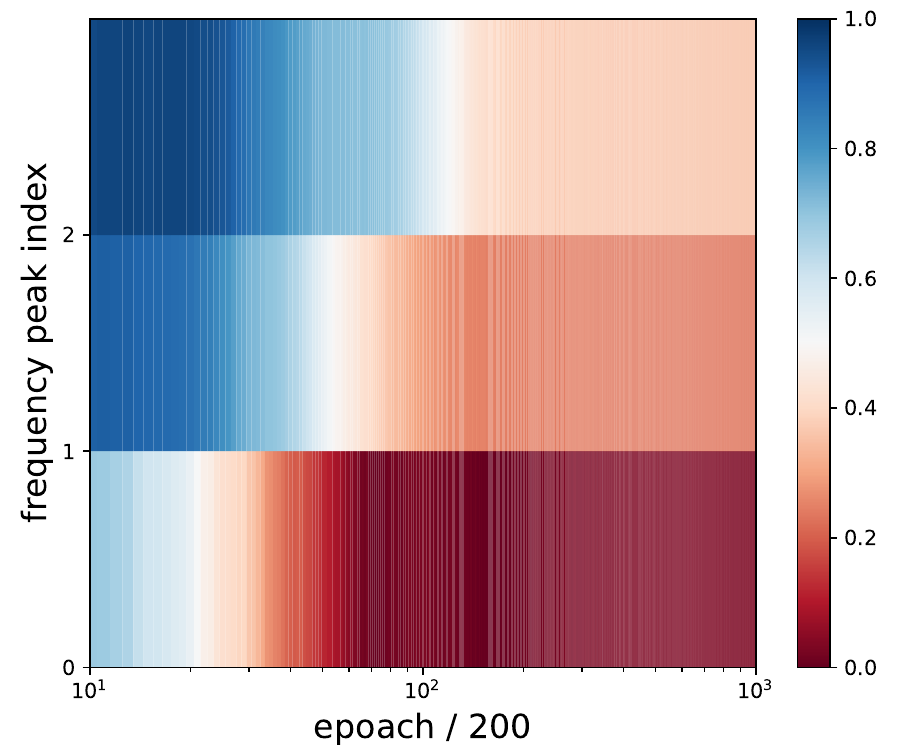}
      \caption{Convergence of TNNs for three frequencies in the $y$ direction.}
      \label{fig:spectral_bias_tnn_d}
  \end{subfigure}
  \caption{\enspace The frequency features of 
  the fitting function (\ref{2d_fit_example})
  and the convergence of TNNs to their frequency features.}
  \label{fig:spectral_bias_tnn}
\end{figure}

We consider the problem of
fitting the following two-dimensional function:
\begin{equation}
  \label{2d_fit_example}
  f(x, y) = \sum_{i=1}^3 \sin(k_i x) + \sin(k_i y),
\end{equation}
where $x,\,y\in [0,\,2\pi]$, $k_1 = 2, \ k_2 = 4, $ and $ k_3 = 6$ represent three 
distinct frequency components in both the $x$- and $y$- directions. 
The Fourier coefficients of $f(x, y)$ are illustrated in Figures~\ref{fig:spectral_bias_tnn_a} and
\ref{fig:spectral_bias_tnn_b}, with the peaks marked at the corresponding frequencies $k_1 = 2, \ k_2 = 4, $ and $ k_3 = 6$. For the case $d=2$, CP-PINNs and TT-PINNs 
yield functionally equivalent representations. Therefore, we adopt a shallow neural
network with architecture [1,100], using a separation rank $r=100$ 
in the CP format. The corresponding TT format uses TT ranks 
$r_0 = r_2 = 1, \ r_1 = 100$, yielding an 
equivalent structure. The activation function is chosen 
as $\tanh(x)$, and training is performed using an initial learning rate of 0.001, decayed exponentially by 
a factor of 0.98 every 1,000 steps.

The evolution of the fitting errors for the three
frequency peaks (low, medium, and high) along both spatial dimensions is shown in Figures~\ref{fig:spectral_bias_tnn_c} and \ref{fig:spectral_bias_tnn_d}.
These results reveal a distinct frequency-dependent convergence pattern: as training progresses, 
lower-frequency components converge first (visualized by the color transition from blue to red), 
followed sequentially by medium-frequency and high-frequency components. This empirical finding directly 
supports the F-Principle, which states that neural networks preferentially learn low-frequency features 
before higher-frequency ones.
Building on these observations and following \cite{xu2018understanding}, we establish
a Fourier-domain theoretical framework to investigate the training dynamics of TNNs. For analytical 
tractability, we focus on a single hidden layer TNNs with $\tanh$ activation function in a two-dimensional 
setting, defined as:
\begin{equation*}
  \Upsilon(x, y) = \sum_{j=1}^r a_j \tanh(w_{x, j}x+b_{x,j}) \cdot \tanh(w_{y, j}y+b_{y,j}).
\end{equation*}

The Fourier transform of $\Upsilon(x, y)$ is expressed as:
\begin{equation*}
  \mathcal{F}[\Upsilon](k_x, k_y) = \sum_{j=1}^r a_j \frac{\pi}{2}\frac{1}{|w_{x,j}w_{y,j}|}\mathrm{e}^{\mathrm{i} \left(\frac{k_xb_{x,j}}{w_{x,j}} + \frac{k_yb_{y,j}}{w_{y,j}}\right)}\left[|w_{x,j}|\delta(k_x) + \frac{\mathrm{i}}{\sinh(\frac{\pi k_x}{2w_{x,j}})}\right]\left[|w_{y,j}|\delta(k_y) + \frac{\mathrm{i}}{\sinh(\frac{\pi k_y}{2w_{y,j}})}\right].
\end{equation*}

We define the spectral error between the TNNs output and target function $f(x)$ as:
\begin{equation*}
  \begin{split}
    D(k_x, k_y) & \triangleq \mathcal{F}[\Upsilon](k_x, k_y) - \mathcal{F}[f](k_x, k_y).
  \end{split}
\end{equation*}
Applying Parseval's theorem, the loss function is equivalent in both the spectral and spatial domains:
\begin{equation*}
  Loss :=\int_{\mathbb{R}^2}L(k_x, k_y)\mathrm{d}k_x \mathrm{d}k_y = \int_{\mathbb{R}^2} \left(\Upsilon(x, y)-f(x, y)\right)^2 \mathrm{d}x \mathrm{d}y,
\end{equation*}
where $L(k_x, k_y)=|D(k_x, k_y)|^2$ and $|\cdot|$ denotes 
the complex modulus.
To enable gradient-based optimization, we compute gradients of the loss with respect to all network parameters: $\Theta_j \triangleq \{w_{x,j}, w_{y,j}, b_{x,j}, b_{y,j}, a_j\}$ at each frequency point $(k_x, k_y)$ in the spectral domain. 

For the case where $w_{x,j}\neq 0,\ w_{y,j} \neq 0$, we express the spectral error in polar form as
\begin{equation*}
    D(k_x, k_y):= |D(k_x, k_y)|\mathrm{e}^{\mathrm{i}\theta(k_x,\, k_y)},
\end{equation*}
and simplify it as:
\begin{equation*}
  \begin{split}
    D(k_x, k_y)  
         &= \sum_{j=1}^r 2\pi a_j \frac{-1}{|w_{x,j}w_{y,j}|} \frac{1}{\sinh\left(\pi k_x/2w_{x,j}\right)}\frac{1}{\sinh\left(\pi k_y/2w_{y,j}\right)} \mathrm{e}^{\mathrm{i} \left(\frac{k_xb_{x,j}}{w_{x,j}} + \frac{k_yb_{y,j}}{w_{y,j}}\right)} - \mathcal{F}[f](k_x, k_y)\\
         &:=\sum_{j=1}^r \frac{a_j}{k_x k_y}D_j\mathrm{e}^{\mathrm{i}\theta_j}- \mathcal{F}[f](k_x, k_y),
  \end{split}
\end{equation*}
where $\theta_j=\frac{k_xb_{x,j}}{w_{x,j}}+\frac{k_yb_{y,j}}{w_{y,j}}$ and $D_j:=D_j(z_x,\,z_y)=-2\pi\frac{|z_x|}{\sinh\left(\pi z_x/2\right)}\frac{|z_y|}{\sinh\left(\pi z_y/2\right)} $ with $(z_x,\,z_y):=\left(\frac{k_x}{w_{x,j}},\frac{k_y}{w_{y,j}}\right)$. If we suppose $k_x/w_{x,j}>0$  and $k_y/w_{y,j}>0$, the gradient of the spectral loss $L(k_x, k_y)$ with respect to $a_j$ is computed as:

\begin{equation}
\label{eq:partailLpartiala}
  \begin{split}
  \frac{\partial L(k_x, k_y)}{\partial a_j} &= \overline{D(k_x, k_y)} \frac{\partial D(k_x, k_y)}{\partial a_j} + D(k_x, k_y)\frac{\partial \overline{D(k_x, k_y)}}{\partial a_j}\\
  &=\frac{|D(k_x, k_y)|D_j}{k_xk_y}\left(\mathrm{e}^{-\mathrm{i}\theta(k_x,\, k_y)}\mathrm{e}^{\mathrm{i}\theta_j}+\mathrm{e}^{\mathrm{i}\theta(k_x,\, k_y)}\mathrm{e}^{-\mathrm{i}\theta_j}\right)\\
  &=\frac{-4\pi}{|w_{x,j}w_{y_j}|}\exp\left(- \left|\frac{\pi k_x}{2w_{x,j}}\right|- \left|\frac{\pi k_y}{2w_{y,j}}\right|\right)
  \frac{1}{1 - \exp\left(-\frac{\pi k_x}{w_{x,j}}\right)}\frac{1}{1 - \exp\left(-\frac{\pi k_y}{w_{y,j}}\right)}{|D(k_x, k_y)|}\cos(\theta_j-\theta(k_x,k_y)).
    \end{split}
\end{equation}

Similar results  hold when $k_x/w_{x,j} < 0$ or $k_y/w_{y,j} < 0$. 
The gradients of the spectral loss with respect to the remaining parameters, 
namely $\frac{\partial L(k_x, k_y)}{\partial w_{x,j}}, \ \frac{\partial L(k_x, k_y)}{\partial w_{y,j}}, \ \frac{\partial L(k_x, k_y)}{\partial b_{x,j}}, $ 
and $ \frac{\partial L(k_x, k_y)}{\partial b_{y,j}}$, can be derived in a similar 
fashion. The corresponding computational details are provided in \ref{cd}. In general, these gradients can be expressed as:
\begin{equation*}
  \label{gradient_tnn_last}
  \left|\frac{\partial L(k_x, k_y)}{\partial \Theta_{jl}}\right| = |D(k_x, k_y)|G_{jl}(\Theta_j,k_x, k_y) H_x\left(\left|\frac{ k_x}{w_{x,j}}\right|\right) H_y\left(\left|\frac{ k_y}{w_{y,j}}\right|\right)\cos(\theta_j-\theta(k_x,k_y)), 
\end{equation*}
where $H_x(z)$ and $H_y(z)$ decay exponentially to zero as the variable $z\rightarrow\infty$, and $0<\underline{\lambda}\leq|G_{jl}(\Theta_j,k_x, k_y)|\leq \overline{\lambda}$. 

The gradient of the total loss $Loss$ with respect to $\Theta_{jl}$ can be expressed as follows:
\begin{equation*}
  \frac{\partial Loss}{\partial \Theta_{jl}}  =\int_{\mathbb{R}^2} \frac{\partial L(k_x, k_y)}{\partial \Theta_{jl}} \mathrm{d}k_x \mathrm{d}k_y \approx
  \sum_{k_x, k_y\in \mathbb{Z}}  \frac{\partial L(k_x, k_y)}{\partial \Theta_{jl}}.
\end{equation*}
To evaluate the contribution of the frequency component $(k_x,\,k_y)$ to the gradient of the total loss, we first present the following lemma.

\begin{lemma}
\label{lem1}
For any bounded positive constants $C_0,\,C_1$, and $C_2\in \mathbb{R}/\{0\},\, C_3\in \mathbb{R}$, it holds that:
\begin{equation*}
  \lim_{\delta\to 0^+}
\frac{
  \mu\left(\left\{\,w\in[-\delta,\delta]:\;\left|\exp(\frac{C_1}{|w|})\cos\left(\frac{C_2}{w} + C_3\right)\right| \geq C_0\right\}\right)
}{2\delta}
=1,
\end{equation*}
where \(\mu(\cdot)\) denotes the Lebesgue measure.
\end{lemma}
\begin{proof}
We first consider $C_2 > 0, \, C_3=0$, the proof can be completed in a similar manner for other situations. 
For $|w|\in[\frac{C_2}{(n+1)\pi},\,\frac{C_2}{n\pi}]$, we have
  \begin{equation*}
    \label{lemma1_eq1}
    \left|\exp(\frac{C_1}{|w|})\cos\left(\frac{C_2}{w} + C_3\right)\right| \geq \exp\left(\frac{C_1}{C_2}n\pi\right) \left|\cos\left(\frac{C_2}{w}\right)\right|= \tilde{c}_0\exp\left(n\right) \left|\cos\left(\frac{C_2}{w}\right)\right|.
  \end{equation*}
Due to $|\sin(x)|\geq |x|/2$ holding for $x\in[-\pi/2,\pi/2]$, we have
  \begin{equation*}
    \left|\cos\left(\frac{C_2}{w}\right)\right|=\left|\sin\left(\frac{C_2}{w}-n\pi-\frac\pi 2\right)\right|\geq \frac 1 2 \left|\frac{C_2}{w}-n\pi-\frac\pi 2\right|.
  \end{equation*}
Let $\xi_n = \frac{C_0}{\tilde{c}_0}\exp\left(-n\right)$.
For $|w|\in[\frac{C_2}{(n+1)\pi},\,\frac{C_2}{(n+1/2)\pi+2\xi_n}]  $ and $[\frac{C_2}{(n+1/2)\pi-2\xi_n},\,\frac{C_2}{n\pi}]$ with $n$ large enough, it follows that:
  \begin{equation*}
    \left|\cos\left(\frac{C_2}{w}\right)\right|\geq\xi_n,
  \end{equation*}
  and
  \begin{equation*}
    \left|\exp(\frac{C_1}{|w|})\cos\left(\frac{C_2}{w} + C_3\right)\right| \geq C_0.
  \end{equation*}
Therefore, it follows: 
\begin{equation*}
    \begin{split}
  \lim_{\delta\to 0^+} &
\frac{
  \mu\left(\left\{\,w\in[-\delta,\delta]:\;\left|\exp(\frac{C_1}{|w|})\cos\left(\frac{C_2}{w} + C_3\right)\right| \geq C_0\right\}\right)
}{2\delta}\\=&\lim_{n_0\to +\infty}
{n_0}\pi\sum\limits_{n=n_0}^\infty\frac{1}{(n+1/2)\pi+2\xi_n}-\frac{1}{(n+1)\pi}+\frac{1}{n\pi}-\frac{1}{(n+1/2)\pi-2\xi_n}\\=&1-\lim_{n_0\to +\infty}
{n_0}\pi\sum\limits_{n=n_0}^\infty\frac{4\xi_n}{(n+1/2)^2\pi^2-4\xi^2_n}=1,
    \end{split}
\end{equation*}
where $n_0 = \lfloor \frac{C_2}{\delta \pi} \rfloor$. The above limitation completes the proof.
\end{proof}

Theorem \ref{thm1} 
generalizes Theorem 1 of \cite{xu2018understanding}, extending
the analysis from standard DNNs to TNNs, 
and from one-dimensional to two-dimensional 
cases. Moreover, using Lemma~\ref{lem1}, we provide a more concise proof to facilitate the extension of this result to higher-dimensional cases.

\begin{theorem}
  \label{thm1}
  Consider two-dimensional TNNs with one hidden layer using tanh function as the activation function and training 
  parameters are all bounded.
  Assume that two bounded frequency points $(k_x^1, k_y^1)$ and $(k_x^2, k_y^2)$ 
  satisfy $k_x^2 > k_x^1 > 0$, $k_y^2 > k_y^1 > 0$
  and $b_{x,j}, \, b_{y,j} \neq 0$. 
  If there exist $c_1, c_2>0$, such that $|D(k_x^1, k_y^1)|>c_1>0$, 
  $|D(k_x^2, k_y^2)|< c_2$, we have  
  \begin{equation}
  \label{theorem:eq-1}
    \lim_{\delta \to 0^+} \frac{\mu \left( \left\{(\left| w_{x,j} \right|,\,\left| w_{y,j}\right|): \left|\frac{\partial L(k_x^1, k_y^1)}{\partial \Theta_{jl}}\right| \geq \left|\frac{\partial L(k_x^2, k_y^2)}{\partial \Theta_{jl}}\right| \quad \text{for all} \quad j,l \right\} \bigcap (0, \delta]^2\right)}{\delta^2} = 1,
  \end{equation}
  where $\Theta_{j}\triangleq \{w_{x,j}, w_{y,j}, b_{x,j}, b_{y,j}, a_j\}$, $\Theta_{jl}\in \Theta_{j}$, and $\mu(\cdot)$ is the Lebesgue measure of a set.
\end{theorem}

\begin{proof}
  First, we prove \eqref{theorem:eq-1} for the case $\Theta_{jl}=a_j$. Let us denote $\hat{H}_x(\Delta k_x)=H_x\left(\left|\frac{ k^2_x}{w_{x,j}}\right|\right)/H_x\left(\left|\frac{ k^1_x}{w_{x,j}}\right|\right)$ with $\Delta k_x=k_{x}^2-k_{x}^1$.
  Due to the exponential decay of $H_x$, we have 
  $$
  \hat{H}_x(\Delta k_x)=C_4\exp\left(\frac{C_1\Delta k_x}{|w_{x,j}|}\right)
  $$
  holding for sufficiently small $|w_{x,j}|$. Similarly, 
  $$
  \hat{H}_y(\Delta k_y)=C_5\exp\left(\frac{C_6\Delta k_y}{|w_{y,j}|}\right).
  $$
According to \eqref{eq:partailLpartiala}, let us define 
\begin{equation*}
    K(k_x^1, k_y^1, k_x^2, k_y^2) = \frac{1-\exp\left(-\frac{\pi k_x^2}{w_{x,j}}\right)}{1-\exp\left(-\frac{\pi k_x^1}{w_{x,j}}\right)}\frac{1-\exp\left(-\frac{\pi k_y^2}{w_{y,j}}\right)}{1-\exp\left(-\frac{\pi k_y^1}{w_{y,j}}\right)}.
\end{equation*}In the case of $k_x^1/w_{x,j}, \, k_y^1/w_{y,j}, \, k_x^2/w_{x,j},$ and $ k_y^2/w_{y,j} > 0$, we have

\begin{equation}
\label{eq:Kkxky}
    0<\underline{\lambda}\leq|K(k^1_x, k^1_y, k^2_x, k^2_y)|\leq \overline{\lambda},
\end{equation}
holding for sufficiently small  $|w_{x,j}|$ and $|w_{y,j}|$. 
Other cases can be treated similarly.

Then, using $|D(k_x^1, k_y^1)|>c_1>0$, $|D(k_x^2, k_y^2)|< c_2 $, and \eqref{eq:Kkxky}, the inequality 
  \begin{equation*}
    \label{eq_tnn2}
    \left|\frac{\partial L(k_x^1, k_y^1)}{\partial a_j}\right| \geq \left|\frac{\partial L(k_x^2, k_y^2)}{\partial a_j}\right|
  \end{equation*} 
  is equivalent to   
  \begin{equation*}
   \exp\left(\frac{C_1\Delta k_x}{|w_{x,j}|}\right)\exp\left(\frac{C_6\Delta k_y}{|w_{y,j}|}\right)\Big|\cos(\frac{k^1_xb_{x,j}}{w_{x,j}}+\frac{k^1_yb_{y,j}}{w_{y,j}}-\theta(k^1_x,\,k_y^1))\Big|\geq C_0.
  \end{equation*} 
  With appropriate modifications, similar results hold for other configurations.
  By applying Lemma~\ref{lem1}, we obtain 
  \begin{equation*}
    \lim_{\delta \to 0^+} \frac{\mu \left( \left\{(\left|w_{x,j}\right|,\,\left|w_{y,j})\right|: \left|\frac{\partial L(k_x^1, k_y^1)}{\partial a_j}\right| \geq \left|\frac{\partial L(k_x^2, k_y^2)}{\partial a_j}\right| \quad \text{for all} \quad j \right\} \bigcap (0, \delta]^2\right)}{\delta^2} = 1.
  \end{equation*}
  The proofs for the remaining parameters follow similarly. This completes the proof of the theorem.
\end{proof}

Theorem~\ref{thm1} indicates that for any two non-converged frequencies, almost all sufficiently small initial weights lead to the prioritization of the lower-frequency component during gradient-based training. This result provides a theoretical explanation for the low-frequency preference observed when fitting the target function \eqref{2d_fit_example}. Furthermore, Theorem~\ref{thm1} can be naturally extended to the $d$-dimensional case, as demonstrated by the following theorem.
\begin{theorem}
  \label{thm2}
  Consider d-dimensional TNNs with one hidden layer using the tanh function as the activation function and the training parameters are bounded.  Assume that bounded frequencies $\bm{k}^1=(k^1_{x_1}, \dots, k^1_{x_d})$, $\bm{k}^2=(k^2_{x_1}, \dots, k^2_{x_d})$ 
  satisfying $k_{x_i}^2 > k_{x_i}^1 > 0$ and 
  parameters $b_{x_i,j}\neq 0$, where $i=1,\dots,d$. If there exist $c_1, c_2>0$ such that $|D(\bm{k}^1)|>c_1>0$, 
  $|D(\bm{k}^2)|< c_2$, we have  
  \begin{equation*}
    \lim_{\delta \to 0^+} \frac{\mu \left( \left\{
    (\left|w_{x_1,j}\right|,\,\ldots,\,\left|w_{x_d,j}\right|): \left|\frac{\partial L(\bm{k}^1)}{\partial \Theta_{jl}}\right| > \left|\frac{\partial L(\bm{k}^2)}{\partial \Theta_{jl}}\right| \quad \text{for all} \quad j,l \right\} \bigcap (0, \delta]^d\right)}{\delta^d} = 1,
  \end{equation*}
  where $\Theta_{j}\triangleq \{a_j,\,w_{x_i,j},\, b_{x_i,j},\,\hbox{with } i=1,\dots,d\}$, 
 and $\mu(\cdot)$ is the Lebesgue measure of a set. 
\end{theorem}

The proof of this theorem follows a similar strategy to 
that of Theorem \ref{thm1}. In the preceding analysis, we have 
provided a theoretical explanation for the low-frequency priority 
phenomenon observed in TNNs when solving high-dimensional problems. 
This behavior arises from the frequency-dependent variation in network 
gradients during training. To address this limitation, the 
next section introduces Fourier feature embeddings as a means 
of enhancing the ability of TNNs in capturing high-frequency 
information in high-dimensional multi-scale problems. 

\begin{remark}
  Theorems 2 and 3 in \cite{xu2018understanding} can also 
  be extended to the TNN framework discussed in this work to 
  further support the F-Principle in high-dimensional 
  settings. However, since the core focus of this paper lies in 
  the development of a frequency-adaptive TNN algorithm, we 
  do not pursue those extensions here.
\end{remark}

\section{Tensor Neural Networks with Fourier Feature}
\label{ff_for_tnn}

For DNNs, various classical techniques have been developed to address the challenges of solving 
multi-scale problems, including the random Fourier feature method \cite{rahimi2007random}, MscaleDNNs 
\cite{liu2020multi, zhang2023correction}. 
Building upon these foundations, we first present the  framework of random Fourier features and 
subsequently generalize its application to TNNs for high-dimensional multi-scale problems. Our 
numerical experiments in this section validate the efficacy of this approach while revealing limitations 
that motivate further development. The following section will introduce targeted enhancements to address these 
current shortcomings.

As established in subsection \ref{fp_tnn}, both DNNs and TNNs exhibit spectral bias, a phenomenon characterized 
by slower convergence rates for high-frequency components during training \cite{rahaman2019spectral}. Recent 
work by \cite{zhong2019reconstructing} demonstrates that simple sinusoidal input transformations can significantly 
enhance a network's ability to learn high-frequency functions. This approach represents a specific instance of 
Fourier feature mapping \cite{tancik2020fourier, rahimi2007random}.
The general formulation of random Fourier feature mapping
$\gamma:\mathbb{R}^d \to \mathbb{R}^{2m}$ is 
given by:
\begin{equation*}
  \gamma[\bm{B}](\bm{x}) = \begin{bmatrix}
    \cos(2\pi \bm{B}\bm{x})\\
    \sin(2\pi \bm{B}\bm{x})
  \end{bmatrix},
  \label{fe2}
\end{equation*}
where the matrix $\bm{B} \in \mathbb{R}^{m\times d}$ 
has entries sampled from Gaussian distribution $\mathfrak{N}(0, \sigma^2)$, with $\sigma > 0$ 
controlling the frequency spectrum. The use of random Fourier 
features renders the Neural Tangent Kernel (NTK) stationary (i.e., shift-invariant), 
effectively endowing it with convolutional properties over the 
input domain \cite{tancik2020fourier}. Moreover, 
this mapping enables explicit control over the NTK's 
bandwidth, thereby improving both the training efficiency 
and generalization ability of the network.

The frequency characteristics of NTK eigenfunctions are determined by the scale parameter $\sigma$ in the 
Gaussian sampling distribution $\mathfrak{N}(0, \sigma^2)$ used for the matrix $\bm{B}$. This introduces an 
important trade-off: 1) larger values of $\sigma$ improve the network's ability to capture high-frequency components, 
helping to mitigate spectral bias; 2) however, excessively large values of $\sigma$ can lead to overfitting and 
degrade performance on functions dominated by low-frequency content. 
To resolve this issue, \cite{wang2021eigenvector} introduced multiple random Fourier feature mappings with 
different 
scales $\sigma_i$, where $i=1, \cdots, \ N_{\sigma}$. This approach provides greater flexibility in capturing a 
wide range of frequencies. The selection of $\sigma_i$ is problem-dependent and 
must be carefully tuned to balance model expressivity and generalization.

To enhance the ability of TNNs to fit high-frequency functions, we incorporate random Fourier feature mappings 
into CP-PINNs or TT-PINNs, as illustrated
in Figure~\ref{rff_tnn}.
For each input variable $x_i$, the corresponding random Fourier feature 
mapping is defined as follows:
\begin{equation}
  \label{ff_eq}
  \gamma[\bm{B}^{(i)}](x_i) = \begin{bmatrix}
    \cos(2\pi \bm{B}^{(i)}x_i)\\
    \sin(2\pi \bm{B}^{(i)}x_i)
  \end{bmatrix},
\end{equation}
where the entries of $\bm{B}^{(i)} \in \mathbb{R}^{m \times 1}, \ i=1,\dots,d$ 
are independently sampled from a Gaussian 
distribution $\mathfrak{N}(0, \sigma_i^2)$, with $\sigma_i > 0$ 
being a user-specified hyperparameter. The transformed variables from equation \eqref{ff_eq} are then used 
as inputs to the subnetworks corresponding to each input dimension.

\begin{figure}[H]
  \centering
  \includegraphics[width=0.6\linewidth]{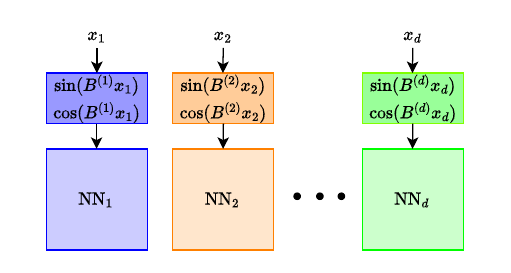}
  \caption{The variables $x_i$ in each dimension are transformed 
  using equation (\ref{ff_eq}) before being fed into the corresponding sub-network. 
  The choice between CP-PINNs and TT-PINNs depends on the 
  distinct output structure and computational processing 
  of each sub-network, which will not be elaborated upon here.}
  \label{rff_tnn}
\end{figure}

Next, we present a simple numerical example to demonstrate that 
the proposed transformation can partially alleviate the challenges 
associated with high-dimensional problems involving high-frequency 
components. Specifically, we consider a 
six-dimensional Poisson equation:
\begin{equation}
  \label{po_eq}
  \begin{split}
    - \Delta u(\bm{x}) &= f(\bm{x}), \quad \bm{x}\in \Omega , \\
    u(\bm{x}) &= g(\bm{x}), \quad \bm{x} \in \partial \Omega,
  \end{split}
\end{equation}
where $\Omega = (0, 1)^6$. 
The functions $f(\bm{x})$ and $g(\bm{x})$ 
are chosen such that the exact solution is
\begin{equation*}
  u_{\text{exact}}(\bm{x}) = \sum_{i=1}^6 \sin(2\pi x_i) + 0.1 \sin(k \pi x_i),
\end{equation*}
where the parameter $k$ controls the high-frequency component of the solution. 
Following the PINN framework \cite{raissi2019physics}, we solve equation (\ref{po_eq}) by minimizing the 
loss function via gradient descent. 
We set the initial learning rate of $0.001$ and 
adopt the activation function
$\varphi(x) = 0.5\sin(x) + 0.5\cos(x)$ \cite{chen2023adaptive}, which has been shown to yield results comparable 
to the sine function while achieving faster convergence \cite{huang2024frequency}.

For comparison, we evaluate four distinct neural network architectures:
\begin{itemize}
  \item CP-PINNs: 
  The architecture is
  illustrated in Figure \ref{tnn_cp}, where each sub-network follows the structure [1, 100, 100, 100].
  
  \item TT-PINNs: As shown in Figure \ref{tnn_tt}, each sub-network is structured as [1, 100, 100, $r_i\times r_{i+1}$] for 
  $i=0, 1, \dots, 5$. The TT ranks are set as 
  $r_0 = 1, \ r_1 = r_2 \dots=r_5=10$, and $r_6 = 1$.

  \item CP-PINNs with Fourier features (CP-PINNs-FF): 
  We set $m=50$ in equation (\ref{ff_eq}), resulting in   
  $\mathbb{B}^{(i)} \in \mathbb{R}^{50\times 1}$ with $\sigma_i=1$ 
  for $i=1,\dots,6$. Each sub-network is configured as [50$\times$ 2, 100, 100]. 
  
  \item TT-PINNs with Fourier features (TT-PINNs-FF): 
  We also set $m=50$ in (\ref{ff_eq}), ensuring that 
  $\mathbb{B}^{(i)} \in \mathbb{R}^{50\times 1}$ with $\sigma_i=1$ 
  for $i=1,\dots,6$. Each sub-network is configured as [50$\times$ 2, 100, $r_i\times r_{i+1}$] for 
  $i=0, 1, \dots, 5$. The TT ranks are set as 
  $r_0 = 1, \ r_1 = r_2 \dots=r_5=10$, and $r_6 = 1$.
\end{itemize}

To ensure a fair comparison between network architectures, CP-PINNs-FF and 
TT-PINNs-FF employ one layer 
less than standard counterparts (CP-PINNs and TT-PINNs) to compensate for the additional parameter overhead 
introduced by the Fourier features. The results for various values of
$k$, which characterize the high-frequency components of the problem, are 
presented in Figure \ref{6d_po_example1}. Due to spectral bias, the relative $L_2$ errors of standard 
TNNs (CP-PINNs and TT-PINNs) remain on the order of $10^{-1}$ for relatively large values of $k$. In contrast, 
incorporating random Fourier features significantly improves the performance of both CP-PINNs and TT-PINNs. In 
this example, CP-PINNs-FF achieves the best overall accuracy among all tested architectures.

This example also highlights that for large values of
frequency parameter $k$, using a fixed
Gaussian distribution $\mathfrak{N}(0, \sigma^2)$ to sample random
frequencies may result in suboptimal Fourier features, thereby limiting 
the accuracy of TNNs to solve multi-scale problems. 
To overcome this limitation, we propose a frequency-adaptive 
method in the following section, which dynamically refines the 
Fourier features to better align with the spectral characteristics 
of the target problem. We further demonstrate that the proposed 
adaptive strategy improves the performance of both CP-based 
and TT-based networks.

\begin{figure}[H]
  \centering
  \includegraphics[width=0.8\linewidth]{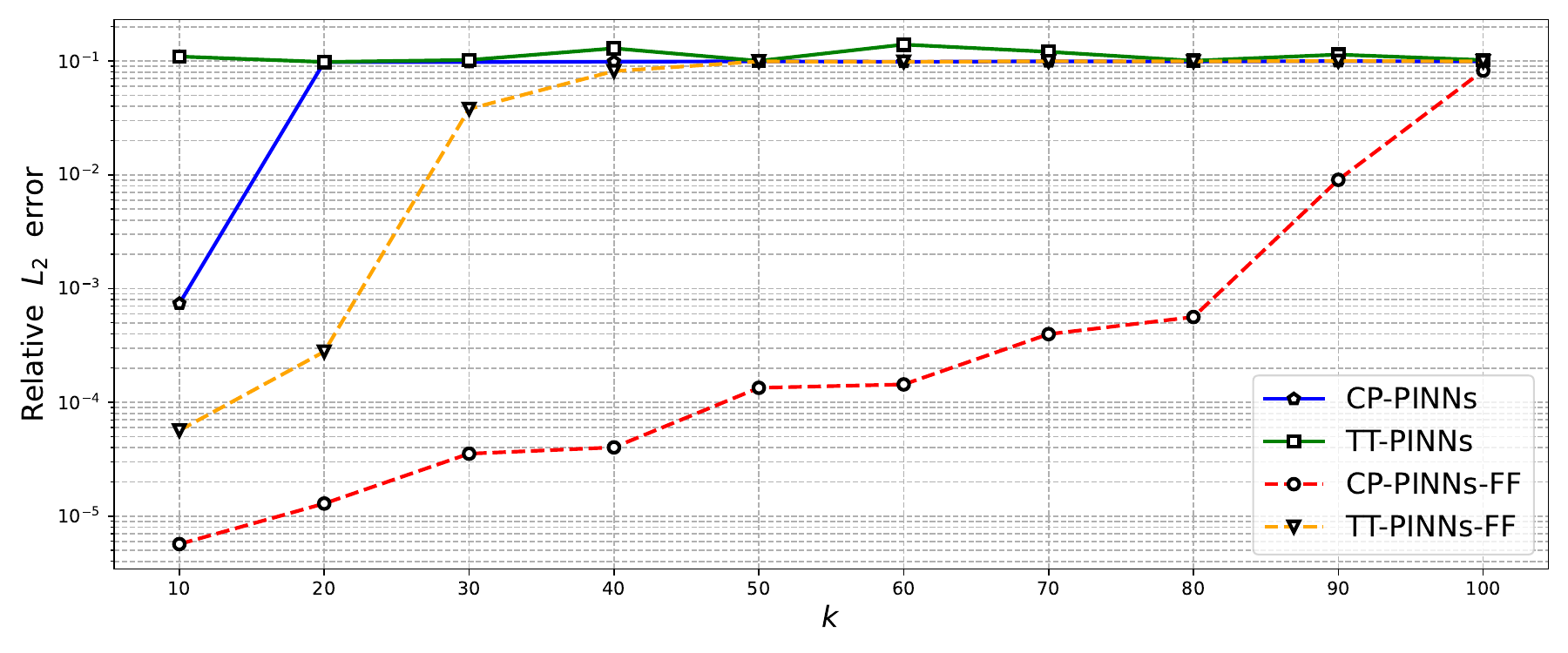}
  \caption{The relative $L_2$ errors for CP-PINNs, TT-PINNs, CP-PINNs-FF, and TT-PINNs-FF with different values of $k$.}
  \label{6d_po_example1}
\end{figure}

\section{Frequency-adaptive Tensor Neural Networks}
\label{fa}

It is straightforward to adopt the frequency-adaptive MscaleDNNs
proposed in \cite{huang2024frequency} to determine the 
frequency features for each input dimension $x_i$ 
within CP-PINNs or TT-PINNs. However, this method relies on the DFT for frequency extraction, 
which becomes
inefficient in high-dimensional settings. Specifically, the
computational cost of the DFT scales as $O(d{\cal N}^d\log {\cal N})$
for a $d$-dimensional problem, where ${\cal N}$ denotes the number of discrete points per dimension. 
To overcome this limitation, we exploit the structural properties of TNNs and propose an efficient frequency 
extraction method that reduces the computational cost from $O(d{\cal N}^d\log {\cal N})$ 
to $O(d{\cal M} {\cal N}\log {\cal N})$, where ${\cal M}$ is a constant independent of $d$. To this end, 
we begin by briefly reviewing the frequency-adaptive MscaleDNNs approach from \cite{huang2024frequency} in 
subsection \ref{fam}, highlighting its limitations due to the reliance on DFT in high-dimensional scenarios. 
We then introduce the proposed frequency-adaptive TNN framework and explain how it effectively addresses 
these challenges.

\subsection{Frequency-adaptive MscaleDNNs}
\label{fam}
Consider a band-limited function $u(\bm{x})$ with 
$\bm{x} \in \mathbb{R}^d$, whose Fourier transform is 
given by: 
\begin{equation*}
  \hat{u}(\bm{k}):=\mathcal{F}[u(\bm{x})](\bm{k}).
\end{equation*} 
Suppose $\hat{u}(\bm{k})$ has a compact support within the domain 
$\mathbb{K}(K_{\text{max}})=\{\bm{k}\in\mathbb{R}^d, \ |\bm{k}|\leq K_{\text{max}}\}$.
Using a down-scaling mappings in phase space, MscaleDNNs \cite{cai2019multi, liu2020multi} decompose $u(\bm{x})$ as: 
\begin{equation*}
  u(\bm{x}) = \sum_{i=1}^N a_i^d u_i(a_i \bm{x}),
\end{equation*}
where $\hat{u}_i(\bm{k}):=\mathcal{F}[u_i(\bm{x})](\bm{k})$ has 
a compact support $\mathbb{K}(\frac{K_{\text{max}}}{a_i})$.
As a result, we can construct networks of the following form
\begin{equation*}
  u_{\text{net}}(\bm{x}; \bm{\theta}) = \sum_{i=1}^N a_i^d u_i(a_i \bm{x}; \bm{\theta}),
\end{equation*}
where each sub-network $u_i(\bm{x};\bm{\theta})$ is designed 
to approximate the corresponding component 
function $u_i(\bm{x})$ at a lower frequency, rather than 
using the full network $u(\bm{x};\bm{\theta})$ to 
directly approximate the original high-frequency 
target function $u(\bm{x})$.

MscaleDNNs utilize a frequency-domain scaling 
strategy that transforms the original input $\bm{x}$ 
into a set of scaled inputs 
$\{a_1 \bm{x}, \ a_2 \bm{x}, \dots, \ a_N \bm{x}\}$. This approach
enables the network to effectively capture multi-scale features 
in target functions $u(\bm{x})$ with rich frequency content. Numerical 
experiments reported in \cite{cai2019multi, liu2020multi} demonstrate 
that MscaleDNNs offer an efficient, mesh-free, and easily implementable 
approach for solving multiscale PDEs. However, their performance
is highly sensitive to the selection of scaling parameters $a_i$. 
To overcome this limitation, frequency-adaptive 
MscaleDNNs \cite{huang2024frequency} 
were proposed. As shown in Algorithm~\ref{alg:FAMscaleDNNs}, these networks dynamically 
optimize the scaling factors
$a_1, \, a_2,\, \dots$ via an adaptive frequency
extraction process, significantly enhancing approximation accuracy. 

\begin{algorithm}[!htpb]
  \caption{Frequency-adaptive MscaleDNNs \cite{huang2024frequency}}
  \label{alg:FAMscaleDNNs}
  \hspace*{0.02in} {\bf Require:}
	Initial feature set $\mathbb{B}_0 = \{2^0, \dots, 2^{N-1}\}$,
	total adaptive steps $I$ and threshold parameter $\rho\in(0,1)$.
  \begin{algorithmic}[1]
    \State \textbf{Initialize:} $u^0_{\text{net}}(\bm{x};\bm{\theta}_0)= \sum_{i=0}^{N-1}2^{i d}\;u^0_i(2^i \bm{x};\bm{\theta}_0)\,.$  
   \For {$ It = 0$ to $I$} 
      \State \textbf{Network Training:} Optimize parameters $\bm{\theta}$ via gradient descent to obtain $u^{It}_{\text{net}}(\bm{x};\bm{\theta}^*_{It})\,.$

      \State \textbf{Frequency Analysis:}
      \State (\romannumeral 1) Compute Fourier coefficients 
        $\hat{u}^{It}_{\bm{k}}$ of $u^{It}_{\text{net}}(\bm{x};\bm{\theta}^*_{It})$ via DFT.
      \State (\romannumeral 2) Extract dominant frequencies: $
          \mathbb{B}_{It+1}
            = \bigl\{\bm{k}_j \mid \bigl|\hat{u}^{It}_{\bm{k}_j}\bigr| > 
              \rho \max_{\bm{k}\in\mathbb{K}} \bigl|\hat{u}^{It}_{\bm{k}}\bigr| 
            \bigr\}. $
     \State {(\romannumeral3) Stopping‐criterion evaluation:
     If $\mathbb{B}_{It} = \mathbb{B}_{It+1}$, terminate the iteration.}

      \State \textbf{Network Adaptation:} $u^{It+1}_{\text{net}}(\bm{x};\bm{\theta}_{It+1})
            = \sum_{j=1}^{J} \hat{u}^{It}_{\bm{k}_j}\;
              u^{It+1}_j\bigl(\bm{k}_j \odot \bm{x};\,\bm{\theta}_{It+1}\bigr)\,. $
    \EndFor
  \end{algorithmic}
    \hspace*{0.02in} {\bf Output:} $u_{\text{net}}^{I}(\bm{x};\bm{\theta}^*_I)$ or $u_{\text{net}}^{It}(\bm{x};\bm{\theta}^*_{It})$ when $\mathbb{B}_{It} = \mathbb{B}_{It+1}$.
\end{algorithm}

The frequency-adaptive MscaleDNNs \cite{huang2024frequency} utilize DFT for frequency extraction. 
While DFT is computationally efficient for low-dimensional problems, it requires uniform mesh 
sampling and incurs a computational cost that scales as $O\left(d{\cal N}^d\log {\cal N}\right)$.
This scaling leads to the curse of dimensionality, making the approach impractical for problems with 
dimensionality greater than four.
This limitation stands in stark contrast to the well-established ability of neural networks to handle 
high-dimensional tasks. To bridge this gap, we propose a novel frequency-adaptive algorithm based on TNNs. 
Our approach preserves computational efficiency while effectively capturing frequency components 
in high-dimensional settings. By leveraging tensor decomposition techniques, the proposed method 
circumvents the dimensional bottlenecks inherent to DFT-based approaches. The details of this 
methodology are presented in the following subsections.

\subsection{Frequency Capture of Tensor Neural Networks}
\label{fc}

In this subsection, we introduce an efficient frequency extraction method for TNNs with 
computational complexity scaling linearly in the dimension $d$, specifically as $O(d{\cal M} {\cal N}\log {\cal N})$. Without loss of
generality, we assume the computational domain $\Omega=[0,1]^d$. 
We begin by defining the {\it  frequency information set } and the {{\it $i$-th dimension frequency 
information set}}.

\begin{definition}
    \label{de1}
    Let $\bm{x}\in [0,1]^d$ and suppose the 
    function $F(\bm{x})$ admits the following Fourier expansion:
    \begin{equation*}
      F(\bm{x}) = \sum_{\bm{k}\in \mathbb{K}} c_{\bm{k}}\mathrm{e}^{2\pi \mathrm{i}\bm{k}\cdot \bm{x}}, \quad c_{\bm{k}}\in \mathbb{C},
    \end{equation*}
    where $\mathbb{K}\subset \mathbb{Z}^d$ is defined as 
    the frequency information set of $F(\bm{x})$, i.e., 
    the set of frequencies with non-negligible spectral 
    coefficients $c_{\bm{k}}$. 
    We define $\mathbb{K}_i$ as the {{\it $i$-th dimension frequency information set}} of 
    $F(\bm{x})$ if and only if
    \begin{equation*}
      k_i' \in \mathbb{K}_i \Leftrightarrow \exists \ (k_1,\dots,\ k_{i-1}, \ k_{i}',\ k_{i+1},\dots,\ k_d) \in \mathbb{K},
    \end{equation*}
    where $k_1, \dots, \ k_{i-1},\ k_{i+1}, \dots, \ k_d \in \mathbb{Z}$.
\end{definition}

\begin{definition}[Cartesian product]
    \label{de2}
    The Cartesian product of two frequency information sets $\mathbb{K}_1$ and $\mathbb{K}_2$, denoted $\mathbb{K}_1\times\mathbb{K}_2$, is defined by:
    $$
    \mathbb{K}_1\times\mathbb{K}_2=\left\{(k_1,\,k_2)\,\big| \,k_1\in \mathbb{K}_1,\, k_2 \in \mathbb{K}_2\right\}.
    $$
\end{definition}

By Definition \ref{de1} and \ref{de2}, it follows that 
\begin{equation}
\label{eq:theorem4.1}
   \mathbb{K} \subset \mathbb{K}_1 \times \cdots \times \mathbb{K}_d.
\end{equation} 
For clarity of exposition, we focus on CP-PINNs in the remainder of this subsection. Specifically, we consider a CP-PINNs with its output defined as: 
\begin{equation}
  \label{fc_eq1}
  F(\bm{x}) := \sum_{\alpha=1}^r f_1(x_1, \alpha)f_2(x_2, \alpha)\cdots f_d(x_d,\alpha).
\end{equation}
All subsequent conclusions can be extended to TT-PINNs in a similar manner.
The following theorem establishes the relationship between the frequency information set
of $F(\bm{x})$ and those of the component functions $f_i(x_i, \alpha)$.

\begin{theorem}
  \label{thm3}
  Let $\bm{x}\in [0,1]^d$ and assume that $F(\bm{x})$ admits a CP decomposition \eqref{fc_eq1}.
  Let $\mathbb{K}$ denote the frequency information set of $F(\bm{x})$, 
  $\mathbb{K}_i$ the frequency information set corresponding to the $i$th-dimension, and $\mathbb{K}_i^{\alpha}$ the frequency information set of $f_i(x_i;\alpha)$, for $i=1,\dots,d$ and 
  $\alpha = 1,\dots,r$. Then the following conclusions hold: 
  \begin{equation*}
    \mathbb{K}_i \subset \bigcup_{\alpha=1}^r \mathbb{K}_i^{\alpha},
  \end{equation*}
  and
  \begin{equation*}
    \mathbb{K} \subset \left(\bigcup_{\alpha=1}^r \mathbb{K}_1^{\alpha}\right)\times \cdots \times \left(\bigcup_{\alpha=1}^r \mathbb{K}_d^{\alpha}\right).
  \end{equation*}
\end{theorem}
\begin{proof}
  We prove the case $i=1$.  The argument
  for $i=2,\dots, d$ follows analogously.
  $\forall k_1\in \mathbb{K}_1$, by Definition \ref{de1},
  this implies the existence of a frequency vector 
  $\bm{k} = (k_1, k_2, \dots, k_d)\in \mathbb{K}$ 
  such that the corresponding Fourier coefficient is nonzero:
  \begin{equation*}
    \begin{split}
      c_{\bm{k}}=\int_{[0,1]^d} F(\bm{x}) \mathrm{e}^{-\mathrm{i}(2\pi  \bm{k}\cdot \bm{x})} \mathrm{d}\bm{x} \neq 0.
    \end{split}
  \end{equation*}
Substituting the CP decomposition of $F(\bm{x})$, we have
\begin{equation*}
    \begin{split}
      c_{\bm{k}}= &\int_{[0,1]^d} \sum_{\alpha=1}^r \prod_{i=1}^d f_i(x_i;\alpha) \mathrm{e}^{-\mathrm{i}(2\pi  k_i x_i)} \mathrm{d}x_1 \mathrm{d}x_2\dots\mathrm{d}x_d \\
      = &\sum_{\alpha=1}^r \left(\int_{[0,1]} f_1(x_1;\alpha) \mathrm{e}^{-\mathrm{i}(2\pi  k_1 x_1)} \mathrm{d}x_1 \right) \left( \int_{[0,1]^{d-1}} \prod_{i=2}^d f_i(x_i;\alpha) \mathrm{e}^{- \mathrm{i}(2\pi  k_i x_i)} \mathrm{d}x_2\dots\mathrm{d}x_d \right).
      \end{split}
  \end{equation*}
      
 Since $c_{\bm{k}}\neq 0$, there must exist some $\alpha\in\{1,\dots,r\}$ such that $$   \int_{[0,1]} f_1(x_1;\alpha)\mathrm{e}^{-\mathrm{i}(2\pi k_1 x_1)} \mathrm{d}x_1 \neq 0.$$
  By Definition \ref{de1}, we have that $k_1\in \mathbb{K}_1^{\alpha}$ for some $\alpha$. Therefore, we conclude that 
  \begin{equation*}
    \mathbb{K}_1 \subset \bigcup_{\alpha=1}^r \mathbb{K}_1^{\alpha},
  \end{equation*}
  which, together with equation \eqref{eq:theorem4.1}, completes the proof.
\end{proof} 

According to Theorem \ref{thm3}, an oversampled frequency set of $F(\bm{x})$ can be constructed by taking the Cartesian product of $\bigcup_{\alpha=1}^r \mathbb{K}_i^{\alpha}$ with $i=1,\,\ldots,\, d$, where each $\mathbb{K}_i^{\alpha}$ is obtained by performing DFT on the single variable function $f_i(x_i, \alpha)$. The total computational complexity of this procedure scales as $O(dr {\cal N}\log {\cal N})$, which is linear in dimension $d$. After generating this oversampled set, we identify the most significant frequencies from it and use 
them to adjust the TNN architecture. While this method successfully circumvents the curse of dimensionality, three critical issues require attention:
\begin{itemize}
  \item[1.] The output of CP-PINNs or TT-PINNs, denoted by {$u(\bm{x}; \bm{\theta})$}, typically serves as
  an approximation to the target function $u(\bm{x})$. As a result, the frequency information sets of target function $u(\bm{x})$ and {$u(\bm{x}; \bm{\theta})$} may differ. Similar to the frequency-adaptive MscaleDNNs proposed in \cite{huang2024frequency}, this discrepancy can be mitigated through iterative refinement.

  \item[2] When applying Theorem \ref{thm3} to construct an oversampled frequency set for the target function $u(\bm{x})$, the resulting frequency information set is only approximate. This approximation poses challenges in accurately identifying the true dominant frequencies of $u(\bm{x})$. To illustrate this issue, consider the following 2D example: $$u(\bm{x})=\sin(x_1)\sin(x_2)+\sin(10x_1)\sin(10x_2)+0.1\sin(10x_1)\sin(x_2).$$
  Here, the true dominant frequencies are $(1,1)$ and $(10,10)$. However, a rank-2 CP decomposition of $u(\bm{x})$ yields:
  $$
  u(\bm{x})=\frac {1}{10}\left[10\sin(x_1)+\sin(10x_1)\right]\left[\sin(x_2)+10\sin(10x_2)\right]-10\sin(x_1)\sin(10x_2).
  $$
  In this representation, the oversampled frequency set may prioritize the frequency pair $(1,10)$, which does not belong to $\mathbb{K}$, the true frequency set of $u(\bm{x})$. Consequently, selecting only the highest-ranked frequencies from the oversampled set could result in the exclusion of the actual dominant frequencies of $u(\bm{x})$. Fortunately, increasing the number of selected frequencies mitigates this issue and improves the chances of correctly capturing the true dominant frequencies. For example, if we pick the two most important frequencies for each component, then the true dominant frequencies are (1, 1) and (10, 10) will be included in the dominant frequency set.  
  
  \item[3] Some frequencies selected from the oversampled frequency set may not belong to the true frequency set $\mathbb{K}$ and can thus be regarded as frequency noise. This raises a natural question: does such noise hinder model training? 
  Interestingly, previous work \cite{wang2021eigenvector} suggests that random Fourier features, essentially a form of frequency noise, can actually benefit training. Their theoretical analysis demonstrates that such noise can be advantageous. 
  Motivated by these findings, we select the dominant frequency set from the oversampled frequency set $\bigcup_{\alpha=1}^r \mathbb{K}_1^{\alpha}\times\cdots \times\bigcup_{\alpha=1}^r \mathbb{K}_d^{\alpha}$, which may include noise, rather than the exact frequency (but computationally expensive) set $\mathbb{K}$. This oversampled frequency set preserves the essential frequency content of the target function while introducing stochastic components.
  Consequently, far from degrading performance, this approach is expected to improve it. Numerical experiments in section~\ref{ne} further substantiates this claim. 

\end{itemize}

\subsection{Frequency-adaptive Tensor Neural Networks}
\label{fa_tnn}
We now present frequency-adaptive TNNs for solving the following high-dimension PDE:
\begin{equation}
  \label{pde}
  \begin{split}
    \mathscr{N}(\bm{x};u(\bm{x}))& = f(\bm{x}), \quad \bm{x}\in \Omega, \\
    \mathscr{B}(\bm{x};u(\bm{x}))& = g(\bm{x}), \quad \bm{x}\in \partial \Omega,
  \end{split}
\end{equation}
where $\Omega=(0,1)^d$,  $\mathscr{N}$ denotes 
a linear or non-linear differential operator, 
$\mathscr{B}$ is the boundary operator, and $u(\bm{x})$ 
is the unknown solution. Following the PINN framework, we reformulate equation~\eqref{pde} as the following soft-constrained optimization problem:
\begin{equation}
  \label{loss1}
  \min\limits_{\bm{\theta}\in \Theta} \mathcal{L}(\bm{\theta}) = \min\limits_{\bm{\theta}\in\Theta}\mathcal{L}_r(\bm{\theta}) + \lambda \mathcal{L}_b(\bm{\theta}),
\end{equation}
where 
\begin{equation*}
  \mathcal{L}_r(\bm{\theta})= \sum_{i=1}^{N_r}\left|\mathscr{N}(\bm{x}^i_r;u_{\text{net}}(\bm{x}^i_r;\bm{\theta}))-f(\bm{x}^i_r)\right|^2 \quad \text{and} \quad \mathcal{L}_b(\bm{\theta})= \sum_{i=1}^{N_b}\left|\mathscr{B}(\bm{x}^i_b;u_{\text{net}}(\bm{x}^i_b;\bm{\theta}))-g(\bm{x}^i_b)\right|^2.
\end{equation*}
Here, $u_{\text{net}}(\bm{x};\bm{\theta})$ represents the output of TNNs, $\Theta$ denotes the parameter space, and $\lambda$ 
is a weighting factor that balances 
the PDE loss $\mathcal{L}_r(\bm{\theta})$ and the boundary 
loss $\mathcal{L}_b(\bm{\theta})$. In equation (\ref{loss1}), 
the sets $\{\bm{x}_r^i\}_{i=1}^{N_r}$ and 
$\{\bm{x}_{b}^i\}_{i=1}^{N_{b}}$ 
represent the sample points within the domain $\Omega$ 
and its boundary $\partial \Omega$, respectively. 
The optimal parameters $\bm{\theta}^*$ are typically obtained by 
minimizing the loss function~(\ref{loss1}) using a stochastic 
gradient-based optimization algorithm. 
For clarity, we adopt the CP decomposition format in our 
symbolic representation, as it offers a more intuitive 
understanding of the proposed approach. The TT
decomposition follows an analogous procedure and can be implemented 
in a similar fashion; hence, we omit a detailed discussion.

Inspired by \cite{huang2024frequency}, we initialize the tensor-based network $u^0_{\text{net}}(\bm{x};\bm{\theta})$ as illustrated in Figure~\ref{rff_tnn}, where the Fourier feature coefficients $\bm{B}^{(i)}$
for each dimension $x_i$ are sampled from the Gaussian distribution $\mathfrak{N}(0, \sigma_i^2)$. Let $\mathbb{B}_{i,0}$ denote the set of all elements in $\bm{B}^{(i)}$, and define the initial frequency information set as $$\mathbb{B}_0=\mathbb{B}_{1,\,0}\times\cdots\times\mathbb{B}_{d,\,0}.  $$
After training the network for $T_0$ steps, 
we obtain a preliminary solution  of the form
\begin{equation*}
  \label{ini_eq}
  u^0_{\text{net}}(\bm{x},\bm{\theta}_0^*) = \sum_{\alpha=1}^r u^0_{1,\alpha}(x_1; \bm{\theta}_0^*)u^0_{2,\alpha}(x_2; \bm{\theta}_0^*)\cdots u^0_{d,\alpha}(x_d; \bm{\theta}_0^*).
\end{equation*}
We then perform DFT on each component $u^0_{i, \alpha}(x_i; \bm{\theta}_0^*)$ using $\mathcal{N}$ uniformly spaced mesh points in the interval $[0,1]$, yielding Fourier coefficients $\hat{u}^0_{i, \alpha, k}$. From these, we select the top
$M$ Fourier coefficients with the largest magnitudes, and denote their corresponding Fourier frequencies as $k_{i,\,\alpha,\,1},\,\ldots,\,k_{i,\,\alpha,\,M}$. The parameter $M$ is chosen in advance to suppress the influence of high-frequency noise and is typically selected from the range of [10, 50]. We define the selected frequency set for each mode as
\begin{equation}
  \label{num_get}
  \mathbb{B}_{i,\,\alpha,\,1} = \{k_{i,\,\alpha,\,1}, \dots, k_{i,\,\alpha,\,M}\}.
\end{equation} 
The aggregated frequency set for the $i$th dimension is then given by
\begin{equation*}
  \label{num_get-1}
 \mathbb{B}_{i,1} = \bigcup_{\alpha=1}^r \mathbb{B}_{i,\,\alpha,\,1}.
\end{equation*} 
Consequently, the refined frequency information set for the tensor-based network becomes
$$\mathbb{B}_1=\mathbb{B}_{1,1}\times\cdots\times \mathbb{B}_{d,1}.$$

Next, we reconstruct TNNs using the refined frequency information set $\mathbb{B}_1$.
For each input variable $x_i$, the Fourier feature mapping is adjusted as: 
\begin{equation*}
  \label{ff_new}
  \gamma[\bm{B}^{(i)}_1](x_i) = \begin{bmatrix}
    \cos(\bm{B}^{(i)}_1 x_i)\\
    \sin(\bm{B}^{(i)}_1 x_i)
  \end{bmatrix},
\end{equation*}
where $\bm{B}^{(i)}_1\in\mathbb{R}^{|\mathbb{B}_{i,\,1}|\times1}$ contains elements drawn from the frequency set $\mathbb{B}_{i,\,1}$. Unlike the initial random Fourier feature mapping, this refined mapping incorporates frequency components extracted from the DFT of the preliminary solution $u^0_{\text{net}}(\bm{x};\bm{\theta}_0^*)$. As a result, it is better aligned with the dominant frequencies of the target function. This frequency adaptation enhances 
the expressiveness of TNNs and enables a more accurate approximation of the multi-scale solution to the given problem~\eqref{pde}.
Following an additional $T_0$ training steps, we obtain the 
updated network approximation $u^1_{\text{net}}(\bm{x};\bm{\theta}_1^*)$. 
We repeat this procedure iteratively to generate the frequency sets $\mathbb{B}_2,\,\mathbb{B}_3,\,\ldots$ and the corresponding network approximations
$u^2_{\text{net}}(\bm{x};\bm{\theta}_2^*)$, $u^3_{\text{net}}(\bm{x};\bm{\theta}_3^*)$, and so on. The iteration is terminated when one of the following stopping criteria is satisfied: 1) the preset maximum number of iterations $I$ is reached; or 2) the frequency set has converged, i.e.,  $\mathbb{B}_{It}=\mathbb{B}_{It+1}$. 
This iterative strategy is referred to as the frequency-adaptive TNN algorithm, and it is summarized in Algorithm~\ref{al}, using the CP format for illustration, with the TT format handled analogously.

\begin{algorithm}[!htpb]
  \caption{Frequency-adaptive Tensor Neural Networks algorithm (CP)}
  \label{al}
  \hspace*{0.02in} {\bf Require:}
  Initial Gaussian distribution $\mathfrak{N}(0, \sigma_i^2)$ 
  for Fourier feature set $\mathbb{B}_0$, total adaptive steps $I$, and parameter $M$ for selecting dominant frequencies.
  \begin{algorithmic}[1]
    \State \textbf{Initialize:} $  u^0_{\text{net}}(\bm{x},\bm{\theta}_0) = \sum_{\alpha=1}^r u^0_{1,\alpha}(x_1; \bm{\theta}_0)u^0_{2,\alpha}(x_2; \bm{\theta}_0)\cdots u^0_{d,\alpha}(x_d; \bm{\theta}_0)\,.$ 
    
    \For {$ It = 0$ to $I$} 

      \State \textbf{Network Training:} Optimize parameters $\bm{\theta}_0$ via gradient descent to obtain $u^{It}_{\text{net}}(\bm{x};\bm{\theta}_{It}^*)\,.$

      \State \textbf{Frequency Analysis:}
      \State (\romannumeral 1) Perform DFT on each component function $u^{It}_{i,\alpha}(x_i; \bm{\theta}_{It}^*)$, a function of one variable, to compute the 
    Fourier coefficients $\hat{u}^{It}_{i, \alpha, k}$.
      \State (\romannumeral 2) Select the $M$ most important frequencies into the set: $\mathbb{B}_{i,\alpha, It+1}.$

      \State (\romannumeral3) Update frequency sets $\mathbb{B}_{i,It+1} = \bigcup_{\alpha=1}^r \mathbb{B}_{i,\,\alpha,\,It+1}$ and $ \mathbb{B}_{It+1}=\mathbb{B}_{1,It+1}\times\cdots\times \mathbb{B}_{d,It+1}$.

    \State {(\romannumeral4) Stopping‐criterion evaluation:
     If $\mathbb{B}_{It} = \mathbb{B}_{It+1}$, terminate the iteration.}
     
       \State \textbf{Sub-network Adjust:} Using Algorithm~\ref{alg:FAMscaleDNNs} to adjust each sub-network $ u^{It+1}_{i,\alpha}(x_i; \bm{\theta}_{It+1})$.

      \State \textbf{Tensor Product:} $  u^{It+1}_{\text{net}}(\bm{x},\bm{\theta}_{It+1}) = \sum_{\alpha=1}^r u^{It+1}_{1,\alpha}(x_1; \bm{\theta}_{It+1})u^{It+1}_{2,\alpha}(x_2; \bm{\theta}_{It+1})\cdots u^{It+1}_{d,\alpha}(x_d; \bm{\theta}_{It+1})\,.$
    \EndFor
  \end{algorithmic}
    \hspace*{0.02in} {\bf Output:} $u_{\text{net}}^{I}(\bm{x};\bm{\theta}^*_I)$ or $u_{\text{net}}^{It}(\bm{x};\bm{\theta}^*_{It})$ when $\mathbb{B}_{It} = \mathbb{B}_{It+1}$.
\end{algorithm}

We then apply the newly proposed frequency-adaptive TNN algorithm to solve Poisson equation \eqref{po_eq} described in section \ref{ff_for_tnn}, and compare its performance with that of CP-PINNs-FF and TT-PINNs-FF. The evaluation considers high-frequency values of $k=60, 70, 80, 90, 100$. The relative $L_2$ errors for all TNNs are shown in
Figure \ref{6d_po_example2}, which illustrates that the frequency-adaptive TNN algorithm achieves a significant reduction in $L_2$ errors
after just one adaptation step.

\begin{figure}[H]
  \centering
  \includegraphics[width=0.8\linewidth]{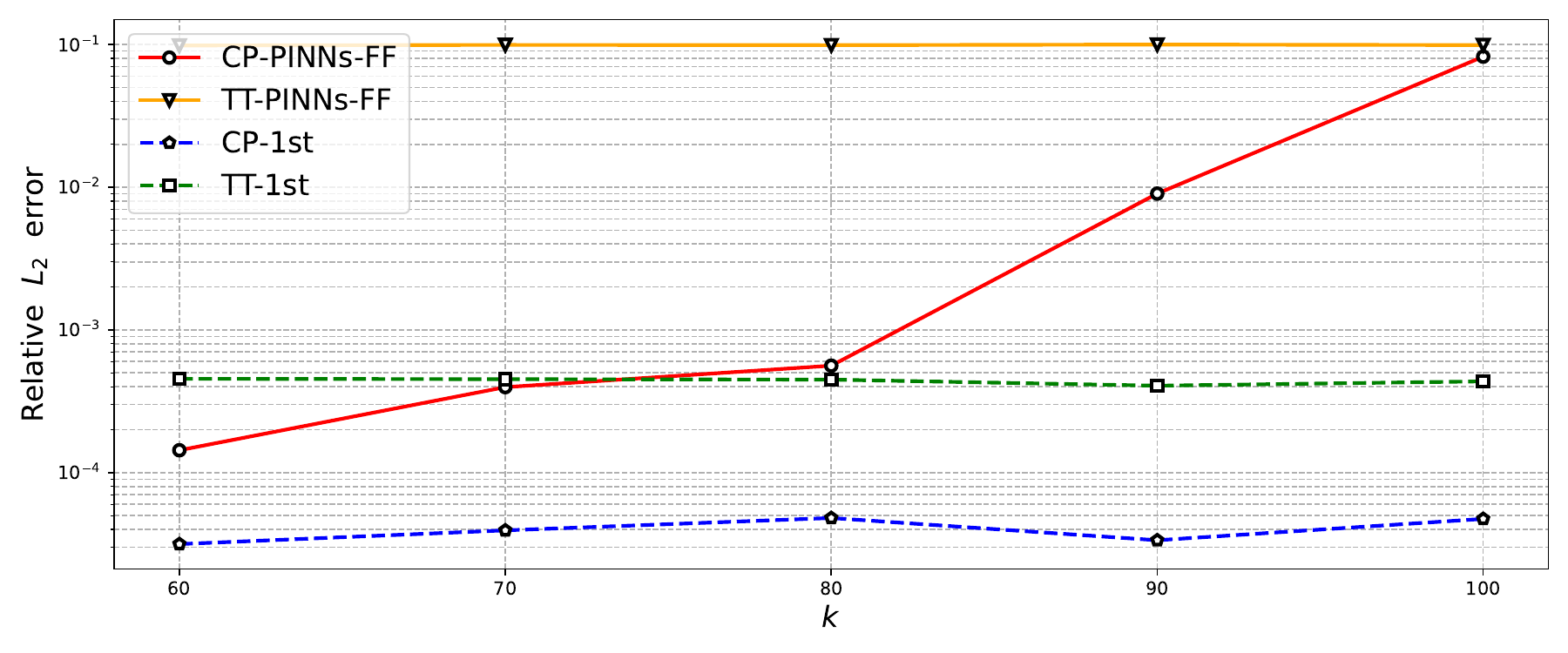}
  \caption{The relative $L_2$ errors for CP-PINNs-FF, TT-PINNs-FF, frequency-adaptive TNN algorithm at $It=1$ based on CP decomposition (denoted as CP-1st) and based on TT decomposition (denoted as TT-1st).}
  \label{6d_po_example2}
\end{figure}

\section{Numerical Examples}
\label{ne}

In this section, we evaluate the performance of the newly proposed
frequency-adaptive TNNs in solving high-dimensional
PDEs with multi-scale features. 
The network weights are initialized using the 
Glorot normal initialization scheme \cite{glorot2010understanding}, 
and all models are trained using the 
Adam optimizer \cite{kingma2014adam} with default settings. Unless 
otherwise specified, the activation function employed is 
$\varphi(x)= 0.5\sin(x)+0.5\cos(x)$. 
The relative $L_2$ error used to evaluate accuracy is defined as 
\begin{equation*}
  E_{\text{rel}}(u_{\text{net}},u_{\text{exact}}) = \frac{\|u_{\text{net}} - u_{\text{exact}} \|_{L_2(\Omega)}}{\| u_{\text{exact}} \|_{L_2(\Omega)}} = \frac{\sqrt{\int_{\Omega} |u_{\text{net}}(\bm{x};\bm{\theta}^*) - u_{\text{exact}}(\bm{x})|^2 \, \mathrm{d}\bm{x}}}{\sqrt{\int_{\Omega} |u_{\text{exact}}(\bm{x})|^2 \, \mathrm{d}\bm{x}}}.
\end{equation*}
where $u_{\text{exact}}$ is the exact solution and $u_{\text{net}}$ 
is the output of frequency-adaptive TNNs.       

Unless stated otherwise, all experiments begin with
the initial network configuration $u_{\text{net}}^0(\bm{x};\bm{\theta})$ 
at iteration $It=0$. For the CP networks, we set
$m=50$ in equation (\ref{ff_eq}), such that the frequency
feature vectors $\bm{B}^{(i)} \in \mathbb{R}^{50\times 1}$ 
are sampled from the Gaussian distribution $\mathfrak{N}(0, \sigma_i^2)$ with
$\sigma_i=10$ for all $i=1,\dots,d$. Each 
sub-network corresponding to the input variable $x_i$ 
follows the architecture [100, 100, 100] in the initial network 
$u_{\text{net}}^0(\bm{x};\bm{\theta})$. 
The final output layer follows the structure in equation~\eqref{cp_eq}, with a weight vector
$W\in \mathbb{R}^{100\times 1}$. 
For the TT networks, we adopt the same settings for 
$m=50$ and $\sigma_i=10$, 
and define each sub-network with architecture 
[100, 100, $r_i\times r_{i+1}$] for 
$i=0, 1, \dots, d$. The TT-rank parameters are set as 
$r_0 = 1, \ r_1 = r_2 \dots=r_{d-1}=10$, and $r_d = 1$. 
The learning rate is initialized at 0.001 and decays 
exponentially by a factor of 0.95 every 1,000 training steps.
At each adaptation step $It$, the neural network solution 
$u_{\text{net}}^{It}(\bm{x};\bm{\theta}_{It}^*)$ 
is obtained by training the network for 
$T_0=100,000$ epochs.

\subsection{Poisson Equation}

We first consider Poisson equation \eqref{po_eq} with $d=3$ and 12, where the computational domain is defined as $\Omega=(0,1)^d$.
For the case $d=3$, we choose appropriate functions $f(\bm{x})$, $g(\bm{x})$ 
in equation (\ref{po_eq}) such that the 
exact solution is

\begin{equation*}
  \begin{split}
    u_{\text{exact}}(\bm{x}) =& \sin(2k_1\pi x_1)\sin(2k_1\pi x_2)\sin(2k_2\pi x_3) + \sin(2k_1\pi x_1)\sin(2k_2\pi x_2)\sin(2k_1\pi x_3)\\
    &+ \sin(2k_2\pi x_1)\sin(2k_1\pi x_2)\sin(2k_1\pi x_3),
  \end{split}
\end{equation*}
where $k_1=10$, $k_2=160$ represent two distinct frequency scales. 
Then, we employ the frequency-adaptive TNNs 
with $I=4$ to solve the multi-scale PDE (\ref{po_eq}). For this example, the TT network uses rank parameters $r_0 = 1, \ r_1 = r_2 = 20$, and $ r_3 = 1$. 
All other initial settings follow the experimental configuration described at the beginning of this section.

\begin{figure}[H]
  \centering
  \begin{subfigure}[b]{0.35\textwidth}
    \includegraphics[width=\textwidth]{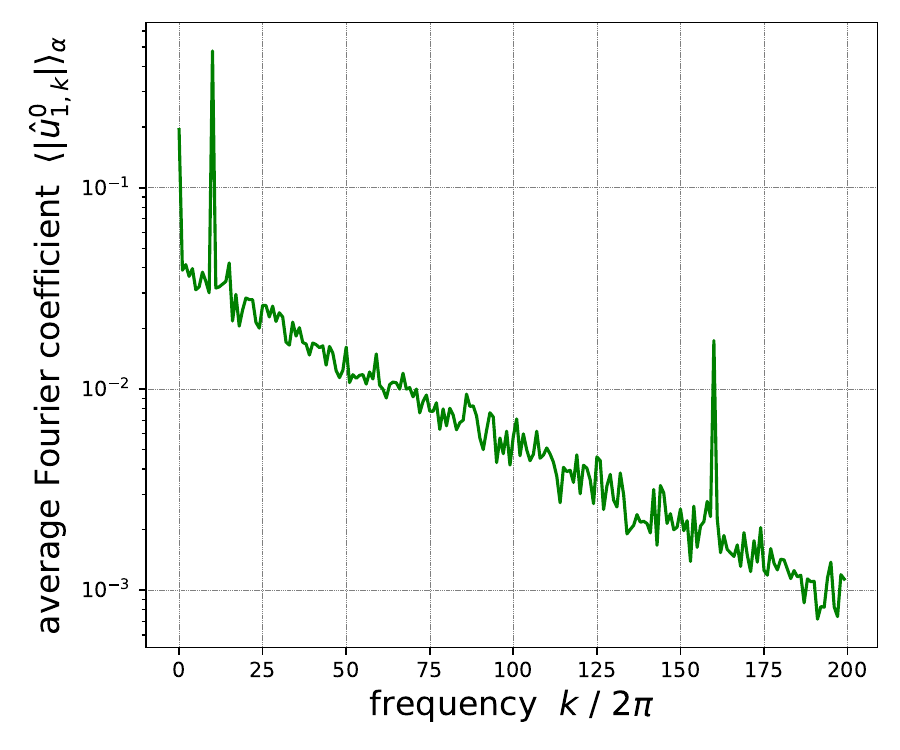}
    \caption{$\langle |\hat{u}^0_{1, k}|\rangle_\alpha$ for CP-PINNs and $It=0$.}
    \label{fig:tnn_3d_po_fre_a}
  \end{subfigure}
  \begin{subfigure}[b]{0.35\textwidth}
    \includegraphics[width=\textwidth]{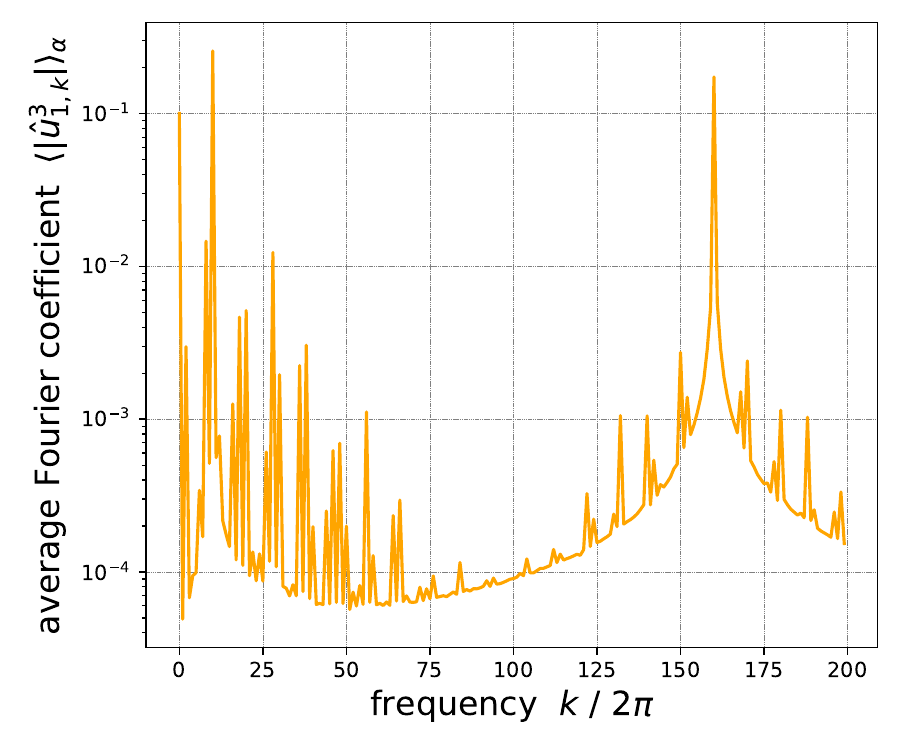}
    \caption{$\langle |\hat{u}^3_{1, k}|\rangle_\alpha$ for CP-PINNs and $It=3$.}
    \label{fig:tnn_3d_po_fre_b}
  \end{subfigure}
  \begin{subfigure}[b]{0.35\textwidth}
      \includegraphics[width=\textwidth]{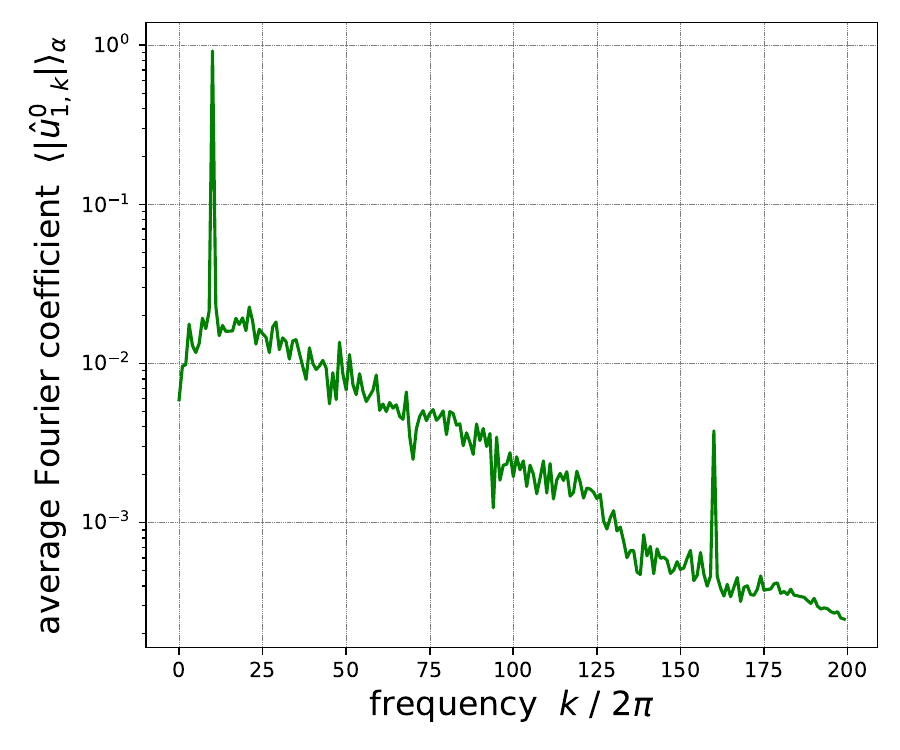}
      \caption{$\langle |\hat{u}^0_{1, k}|\rangle_\alpha$ for TT-PINNs and $It=0$.}
      \label{fig:tnn_3d_po_fre_c}
  \end{subfigure}
  \begin{subfigure}[b]{0.35\textwidth}
      \includegraphics[width=\textwidth]{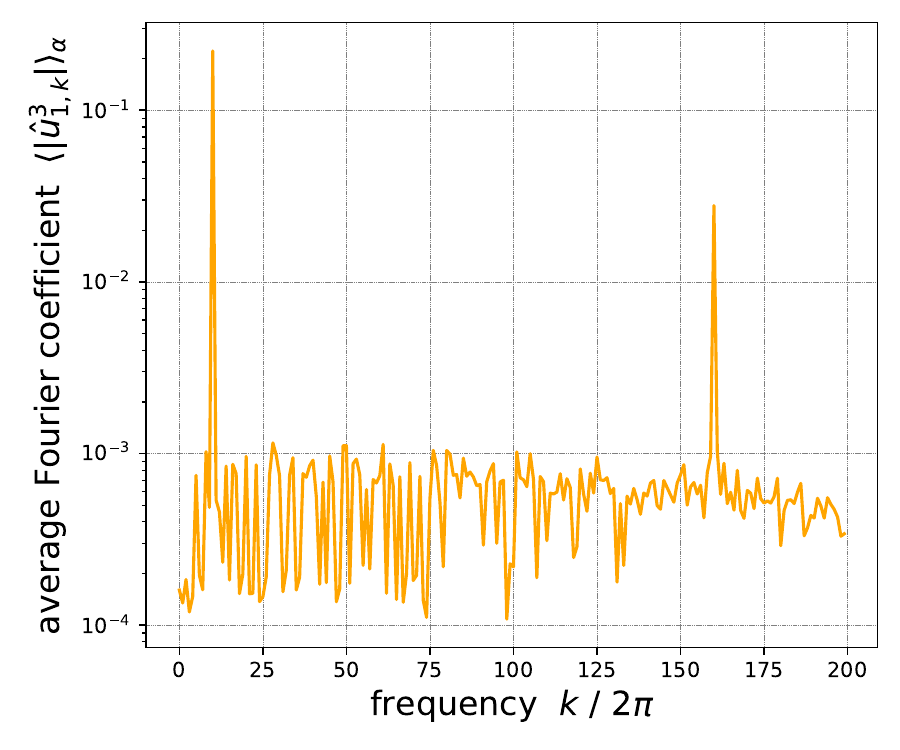}
      \caption{$\langle |\hat{u}^3_{1, k}|\rangle_\alpha$ for TT-PINNs and $It=3$.}
      \label{fig:tnn_3d_po_fre_d}
  \end{subfigure}
  \caption{\enspace  
  The distributions of $\langle| \hat{u}^{It}_{1, k}|\rangle_\alpha$ for Poisson equation (\ref{po_eq}) with $d=3$.}
  \label{fig:tnn_3d_po_fre}
\end{figure}

The relative $L_2$ errors at each adaptive iteration are reported in 
Table \ref{table_po}. At the initial step $It = 0$, the 
relative $L_2$ errors for CP-PINNs and TT-PINNs, 
both employing random Fourier features,
are 9.447e-02 and 7.381e-01, respectively. 
These results indicate that the 
initial approximation $u^0_{\text{net}}(\bm{x};\bm{\theta}_0^*)$ 
lacks high accuracy. For CP-PINNs, we then perform DFT on each component of $u^0_{\text{net}}(\bm{x};\bm{\theta}_0^*)$, to obtain $\hat{u}^0_{i, \alpha, k}$. 
To clearly illustrate the frequency distribution, we define the average frequency amplitude across all components in the 
$x_i$-direction at $It$-th adaptive step as:
\begin{equation*}
  \langle |\hat{u}^{It}_{i,k}|\rangle_\alpha = \frac{1}{r}\sum_{\alpha=1}^r|\hat{u}^{It}_{i, \alpha, k}|.
\end{equation*}
A similar definition applies to TT-PINNs.
Figures \ref{fig:tnn_3d_po_fre_a} and \ref{fig:tnn_3d_po_fre_c} visualize the distribution of $\langle |\hat{u}^{It}_{i,k}|\rangle_\alpha$ for CP-PINNs and TT-PINNs at $It=0$, respectively.
Despite the inaccuracy of $u^0_{\text{net}}(\bm{x};\bm{\theta}_0^*)$, the network effectively identifies key frequency features: the coefficients at $k=10$ and $k = 160$ are markedly larger than neighboring frequencies. 

In this test case, the parameter $M$ in equation 
(\ref{num_get}) is set to $10$, representing the number of selected frequencies. 
Using Algorithm~\ref{al}, we construct the frequency feature set $\mathbb{B}_1$,
which is then used to initialize the updated network 
$u_{\text{net}}^1(\bm{x};\bm{\theta})$. 
After retraining, we obtain the refined approximation 
$u_{\text{net}}^1(\bm{x};\bm{\theta}_1^*)$. 
The incorporation of posterior frequency adaptation 
significantly improves the accuracy, reducing the relative 
$L_2$ error to 6.155e-03 for CP-PINNs and 1.583e-02 for TT-PINNs. 
These results highlight the effectiveness of the frequency-adaptive strategy after just a single iteration.

Due to the presence of frequency noise, the termination condition 
$\mathbb{B}_{It}=\mathbb{B}_{It+1}$ 
is difficult to satisfy exactly. As a result, the frequency 
adaptation process is terminated after a fixed number of 
iterations, specifically at $It=4$. 
The frequency distributions obtained 
from CP-PINNs and TT-PINNs just before the final training step are shown in
Figures \ref{fig:tnn_3d_po_fre_b} and \ref{fig:tnn_3d_po_fre_d}, 
respectively. Compared to those from the initial adaptation, 
the frequency components in the final step are more 
sharply concentrated around $k=10$ and $k=160$, indicating 
enhanced frequency alignment and more effective feature extraction. 
As frequency estimation becomes more accurate over successive iterations, the solution 
accuracy also improves. The final relative $L_2$ 
errors are reduced to 3.850e-03 for CP-PINNs and 1.490e-03 for TT-PINNs. 
These results 
confirm the effectiveness and robustness of the 
proposed frequency-adaptive TNN framework.

\begin{figure}[H]
  \centering
  \begin{subfigure}[b]{0.35\textwidth}
    \includegraphics[width=\textwidth]{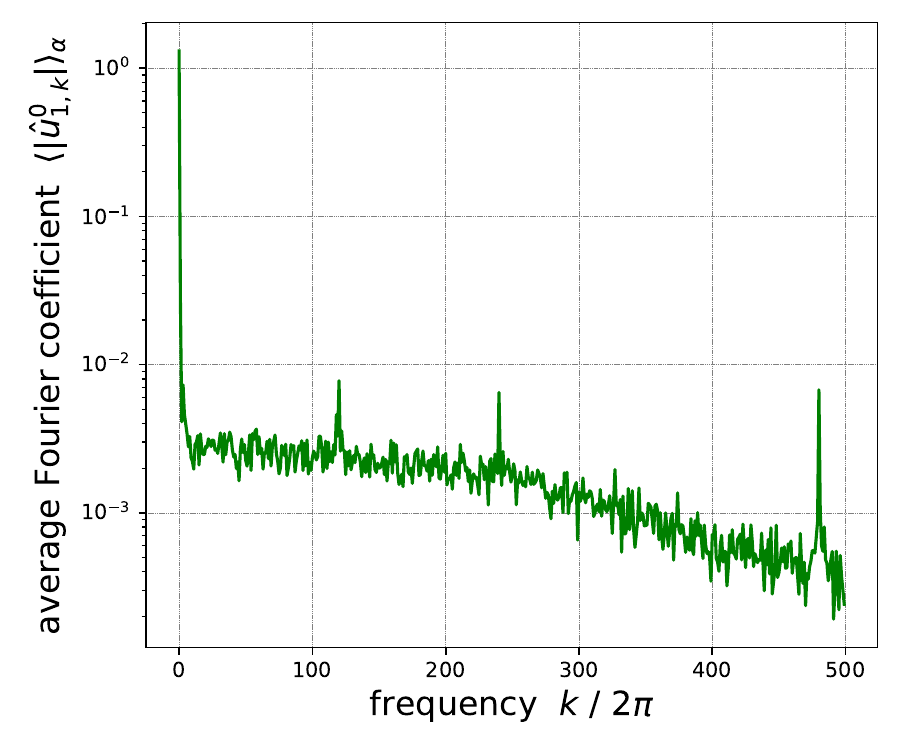}
    \caption{$\langle |\hat{u}^0_{1, k}|\rangle_\alpha$ for CP-PINNs and $It=0$.}
    \label{fig:tnn_12d_po_fre_a}
  \end{subfigure}
  \begin{subfigure}[b]{0.35\textwidth}
    \includegraphics[width=\textwidth]{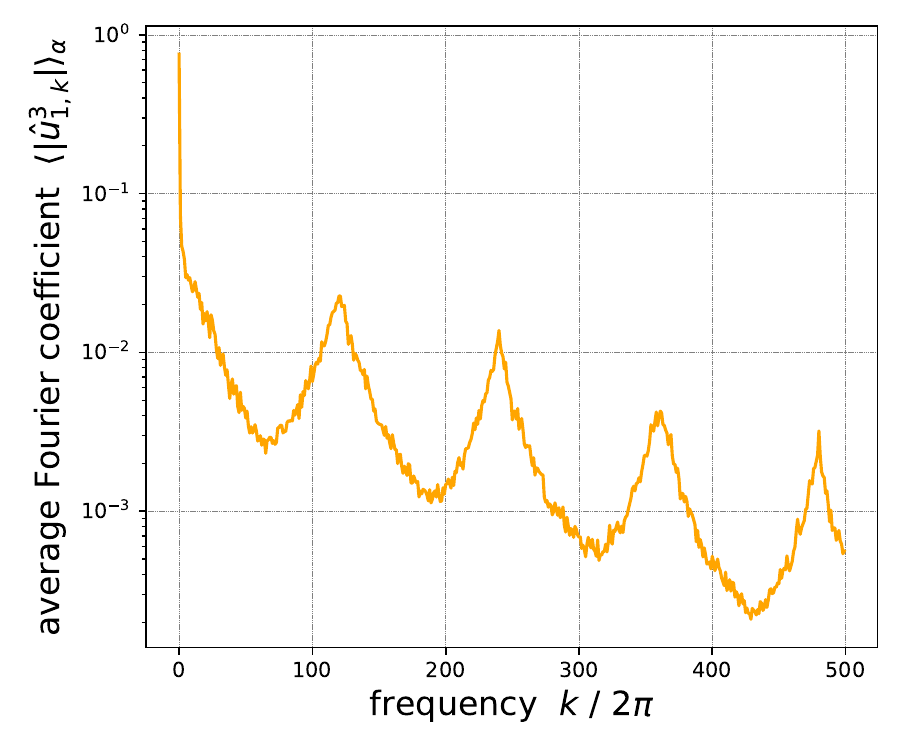}
    \caption{$\langle |\hat{u}^3_{1, k}|\rangle_\alpha$ for CP-PINNs and $It=3$.}
    \label{fig:tnn_12d_po_fre_b}
  \end{subfigure}
  \begin{subfigure}[b]{0.35\textwidth}
      \includegraphics[width=\textwidth]{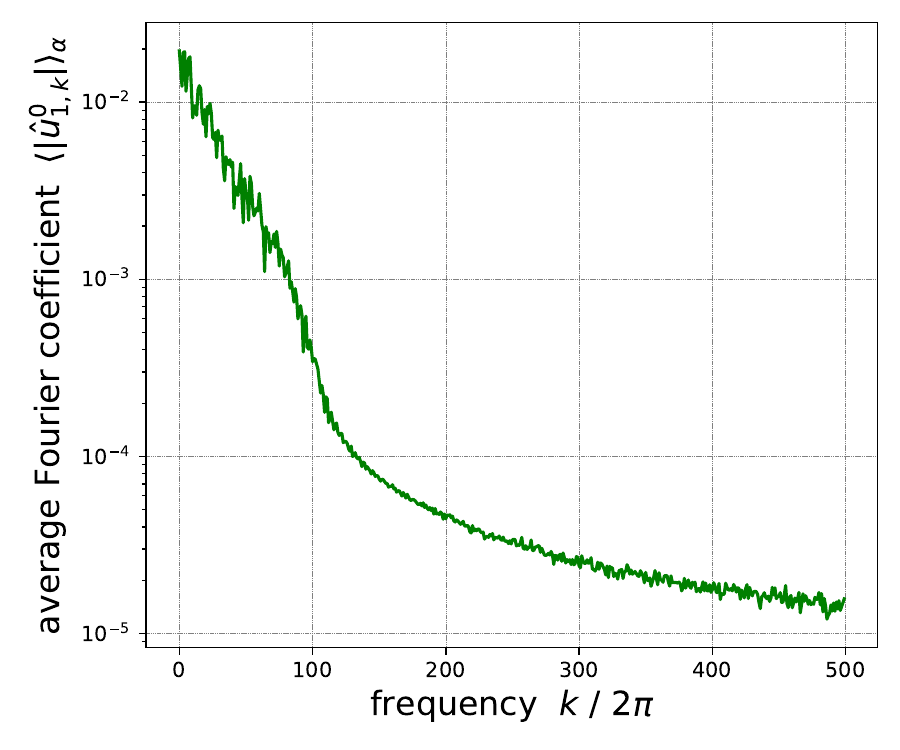}
      \caption{$\langle |\hat{u}^0_{1, k}|\rangle_\alpha$ for TT-PINNs and $It=0$.}
      \label{fig:tnn_12d_po_fre_c}
  \end{subfigure}
  \begin{subfigure}[b]{0.35\textwidth}
      \includegraphics[width=\textwidth]{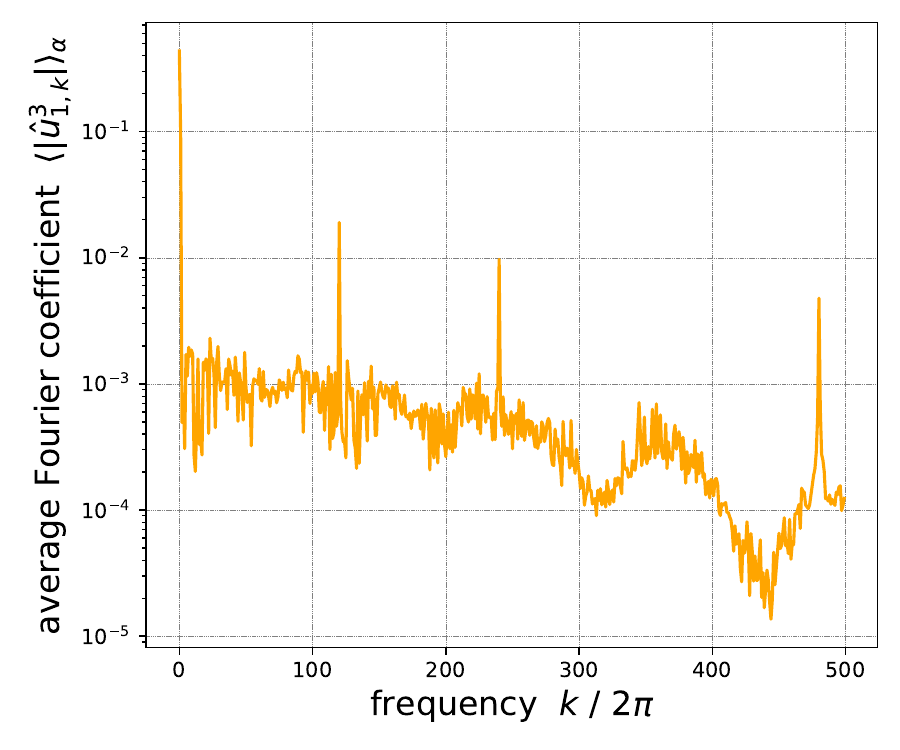}
      \caption{$\langle |\hat{u}^3_{1, k}|\rangle_\alpha$ for TT-PINNs and $It=3$.}
      \label{fig:tnn_12d_po_fre_d}
  \end{subfigure}
  \caption{\enspace The distributions of $\langle |\hat{u}^{It}_{1, k}|\rangle_\alpha$ for Poisson equation (\ref{po_eq}) with $d=12$. }
  \label{fig:tnn_12d_po_fre}
\end{figure}

Next, we consider a more challenging example involving 
higher dimensionality and more intricate frequency features. 
Specifically, we set $d=12$, $\Omega=(0,1)^{12}$ and choose 
appropriate functions $f(\bm{x})$ and $g(\bm{x})$ 
in the Poisson equation (\ref{po_eq}) such that 
the exact solution is given by
\begin{equation*}
  u_{\text{exact}}(\bm{x}) = \sum_{i=1}^{12} \sin(2k_1\pi x_i) + \sin(2k_2\pi x_i) + 0.1\sin(2k_3\pi x_i) + 0.05\sin(2k_4\pi x_i),
\end{equation*}
where $k_1=1, \ k_2=120, \, k_3=240, $ and $ k_4 = 480$.
To numerically solve this problem, we employ 
frequency-adaptive TNNs with 
$I=4$ adaptation steps. All initial settings are kept identical 
to those specified in the experimental configuration at 
the beginning of this section.         

The relative $L_2$ errors at each adaptive iteration are reported in 
Table \ref{table_po}. In the initial step ($It = 0$), the 
relative $L_2$ errors for CP-PINNs and TT-PINNs
are 2.360e-02 and 9.995e-01, respectively. 
To evaluate the point-wise distribution of error distribution, we randomly sample 
1,600 points in $\Omega$ and map the corresponding errors onto a $40\times 40$ 2D array. Note that this mapping is used purely
for visualization purposes and does not correspond to 
the actual spatial coordinates of the sampled points.  
At $It=0$, the point-wise errors
for CP-PINNs and TT-PINNs are illustrated in
Figure \ref{fig:tnn_12d_po_error_a} and 
Figure \ref{fig:tnn_12d_po_error_c}, respectively, 
with error magnitudes ranging approximately from $10^{-1}$ to $1$. These results clearly show that the initial approximation $u^0_{\text{net}}(\bm{x};\bm{\theta}_0^*)$ 
is not very accurate.
We then
perform the DFT on each component $u^0_{i,\alpha}(x_i;\bm{\theta}_0^*)$, and examine the average frequency coefficients. 
The results in the $x_1$-direction are presented in Figures~\ref{fig:tnn_12d_po_fre_a} and~\ref{fig:tnn_12d_po_fre_c} for CP-PINNs and TT-PINNs, respectively. 
The frequency spectrum of CP-PINNs shows clear peaks at $k=1, 120, 240, 480$,
whereas the TT-PINNs exhibit a smoother frequency distribution.
Based on these observations, we set the parameter in equation~(\ref{num_get}) as $M=5$ for CP-PINNs and $M=10$ for TT-PINNs. 

\begin{figure}[H]
  \centering
  \begin{subfigure}[b]{0.35\textwidth}
    \includegraphics[width=\textwidth]{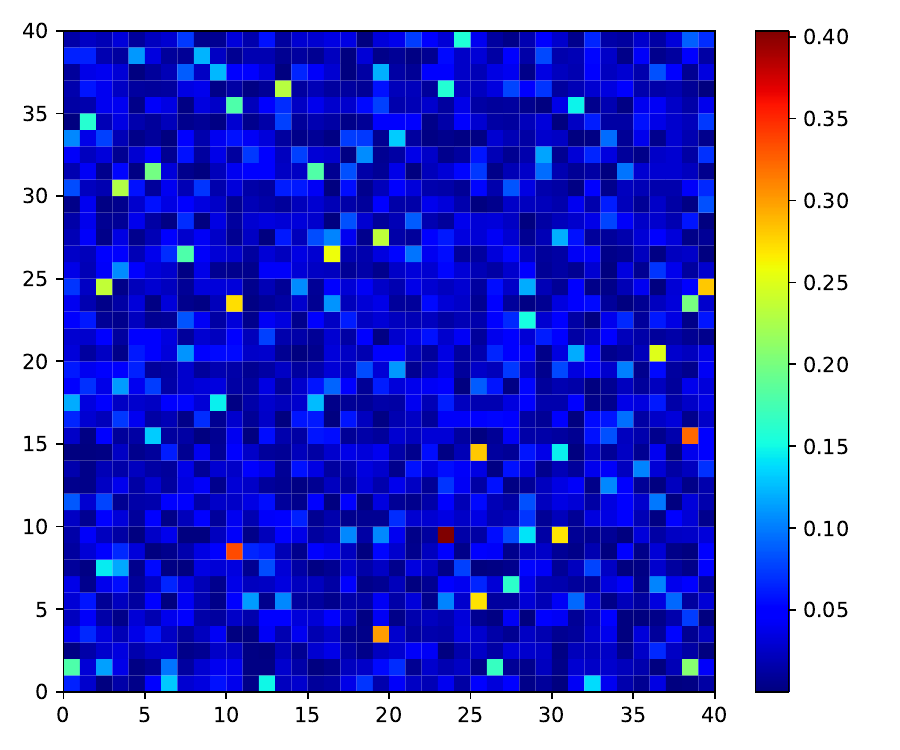}
    \caption{Point-wise errors for CP-PINNs and $It=0$.}
    \label{fig:tnn_12d_po_error_a}
  \end{subfigure}
  \begin{subfigure}[b]{0.35\textwidth}
    \includegraphics[width=\textwidth]{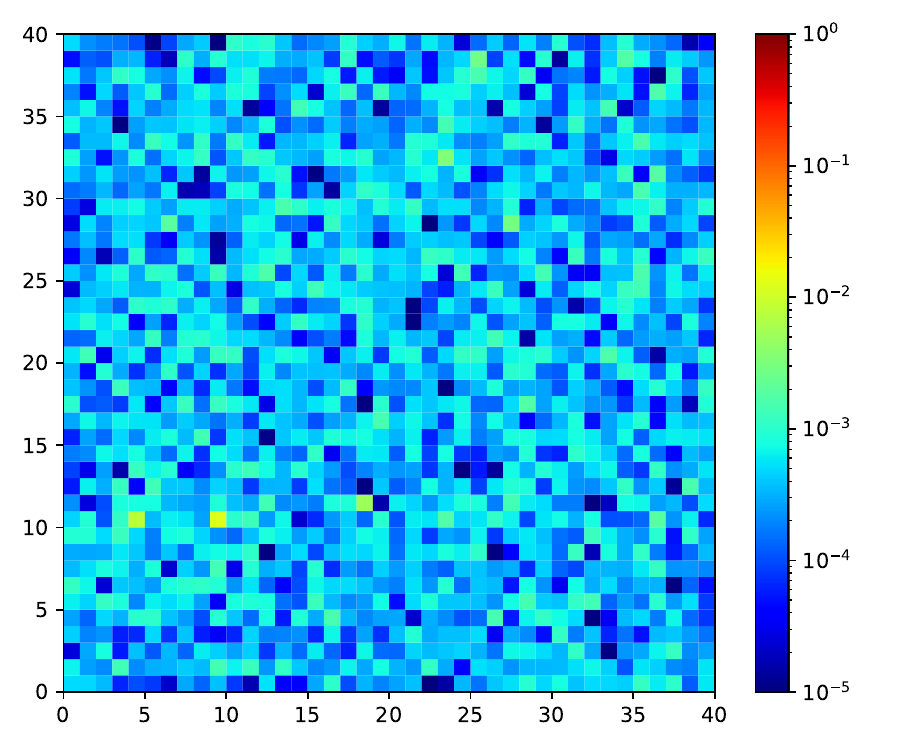}
    \caption{Point-wise errors for CP-PINNs and $It=4$.}
    \label{fig:tnn_12d_po_error_b}
  \end{subfigure}
  \begin{subfigure}[b]{0.35\textwidth}
      \includegraphics[width=\textwidth]{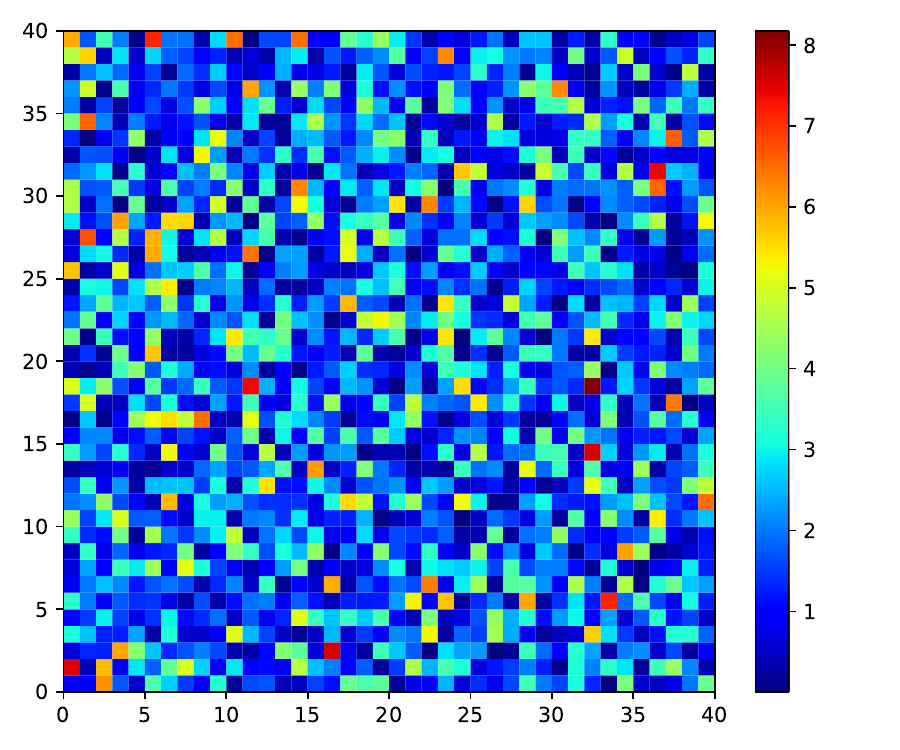}
      \caption{Point-wise errors for TT-PINNs and $It=0$.}
      \label{fig:tnn_12d_po_error_c}
  \end{subfigure}
  \begin{subfigure}[b]{0.35\textwidth}
      \includegraphics[width=\textwidth]{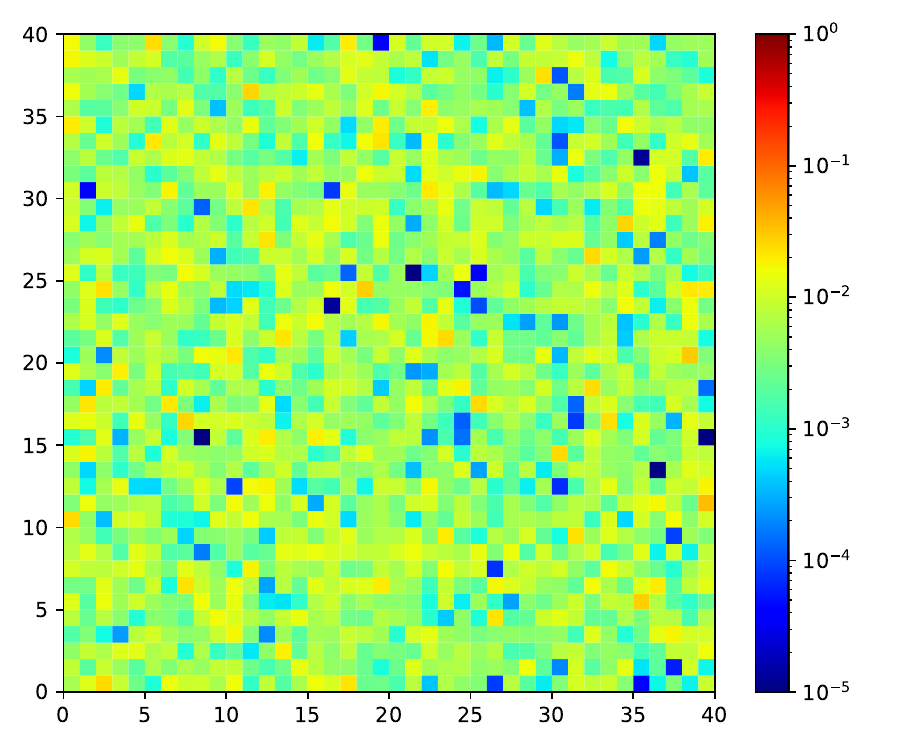}
      \caption{Point-wise errors for TT-PINNs and $It=4$.}
      \label{fig:tnn_12d_po_error_d}
  \end{subfigure}
  \caption{\enspace Point-wise errors for 
  Poisson equation (\ref{po_eq}) with $d=12$.}
  \label{fig:tnn_12d_po_error}
\end{figure}

Using Algorithm~\ref{al}, we construct the frequency feature set $\mathbb{B}_1$,
which is then used to initialize the updated network 
$u_{\text{net}}^1(\bm{x};\bm{\theta})$. 
After retraining, we obtain the approximation $u^1_{\text{net}}(\bm{x};\bm{\theta}_1^*)$ in the first adaptive step.
For CP-PINNs, since the dominant frequencies are selected in $\mathbb{B}_1$, the relative $L_2$ is reduced to 2.088e-04, substantially lower than that of TT-PINNs (6.702e-02). 
As the iterations continue, the frequency-adaptive TT-PINNs gradually capture all key frequency components, as shown in Figure~\ref{fig:tnn_12d_po_fre_d}.  
By the final iteration, the
relative $L_2$ errors are further reduced to 2.825e-04 for CP-PINNs and 3.338e-03 for TT-PINNs. 
The corresponding point-wise errors at 1,600 randomly sampled points, 
illustrated in Figure \ref{fig:tnn_12d_po_error_b} 
and Figure \ref{fig:tnn_12d_po_error_d}, 
demonstrate a substantial reduction, by several 
orders of magnitude, compared to that at $It=0$. 
These results underscore that, although TT-PINNs may perform poorly in the early stages, the proposed frequency-adaptive algorithm significantly enhances their performance, ultimately achieving high accuracy.

\begin{table}[H]
  \begin{center}
  \caption{Relative $L_2$ error at each adaptive step for Poisson equation (\ref{po_eq}).}
  \begin{tabular}{cccccc}
    \hline
    Adaptive step ${It}$ & 0 & 1 & 2 & 3 & 4\\
    \hline 
    CP-PINNs: $d=3$ & 9.447e-02 & 6.155e-03 & 3.279e-03& 3.920e-03 & 3.850e-03 \\
    \hline
    TT-PINNs: $d=3$ & 7.381e-01 & 1.583e-02 & 2.949e-03& 3.879e-03 & 1.490e-03 \\
    \hline
    CP-PINNs: $d=12$ & 2.360e-02 & 2.088e-04 & 4.267e-04& 4.051e-04 & 2.825e-04 \\
    \hline
    TT-PINNs: $d=12$ & 9.995e-01 & 6.702e-02 & 4.535e-03& 3.202e-03 & 3.338e-03 \\
    \hline
    \label{table_po}
  \end{tabular}
\end{center}
\end{table}

\subsection{Heat Equation}

To evaluate the performance of the proposed 
frequency-adaptive TNNs in 
solving spatio-temporal multi-scale problems, 
we consider a six-dimensional heat equation:
\begin{equation}
  \label{tnn_heat_eq}
  \begin{split}
    u_t(\bm{x}, t)& = \sum_{i=1}^6 \frac{1}{(k_i \pi)^2} u_{x_i x_i}(\bm{x},t), \quad (\bm{x},t)\in \Omega\times (0, T], \\
    u(\bm{x}, 0) & = h(\bm{x}), \quad  \bm{x} \in \Omega, \\
    u(\bm{x}, t) & = g(\bm{x}), \quad  (\bm{x}, t) \in \partial \Omega \times (0, T],
  \end{split}
\end{equation}
where $\Omega=(0, 1)^6$, $T=1$, and appropriate functions
$h(\bm{x})$ and $g(\bm{x})$ are chosen such that the 
exact solution $u(\bm{x}, t)$ is given by:
\begin{equation*}
  u_{\text{exact}}(\bm{x}, t) = \exp(-t)\left(\sum_{i=1}^3 \left(\sin(k_{2i-1}x_{2i-1}) + \cos(k_{2i}x_{2i})\right)\right),
\end{equation*}
with $k_1=k_2=k_3=200$, $k_4=k_5=k_6=400$. 
We employ the frequency-adaptive TNNs with $I=4$
adaptive iterations to numerically solve 
the multi-scale PDE (\ref{tnn_heat_eq}). 

\begin{figure}[H]
  \centering
  \begin{subfigure}[b]{0.35\textwidth}
    \includegraphics[width=\textwidth]{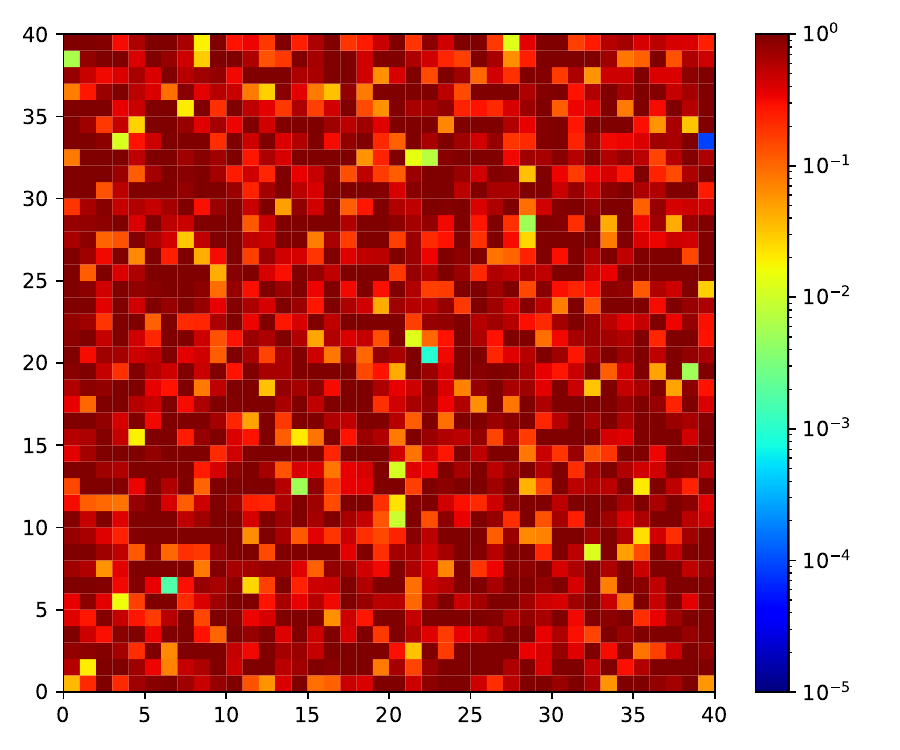}
    \caption{Point-wise errors for CP-PINNs and $It=0$.}
    \label{fig:tnn_6d_heat_error_a}
  \end{subfigure}
  \begin{subfigure}[b]{0.35\textwidth}
    \includegraphics[width=\textwidth]{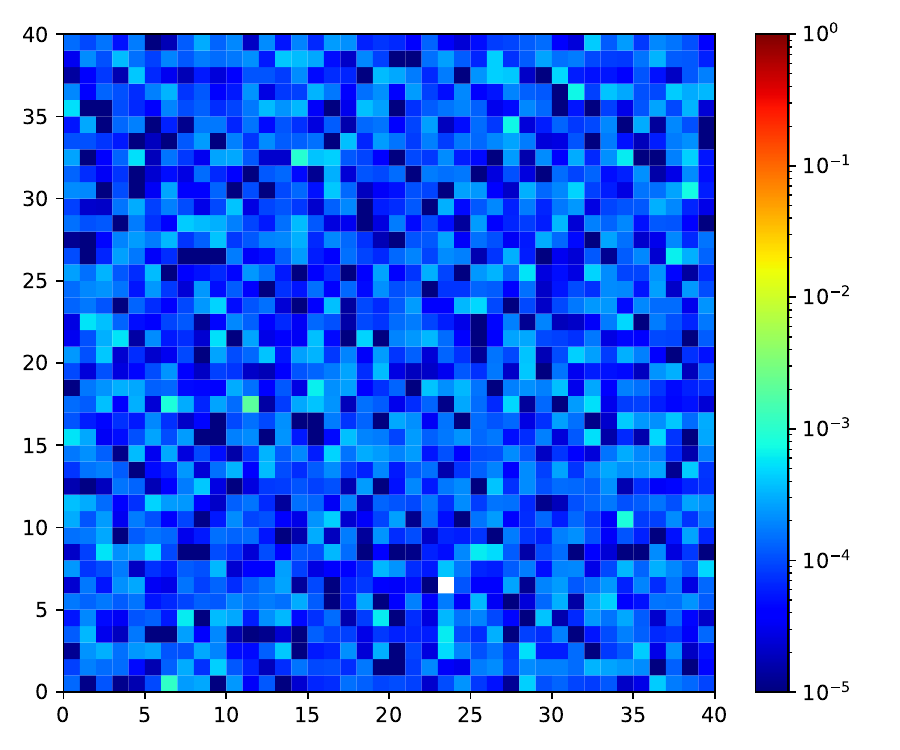}
    \caption{Point-wise errors for CP-PINNs and $It=4$.}
    \label{fig:tnn_6d_heat_error_b}
  \end{subfigure}
  \begin{subfigure}[b]{0.35\textwidth}
      \includegraphics[width=\textwidth]{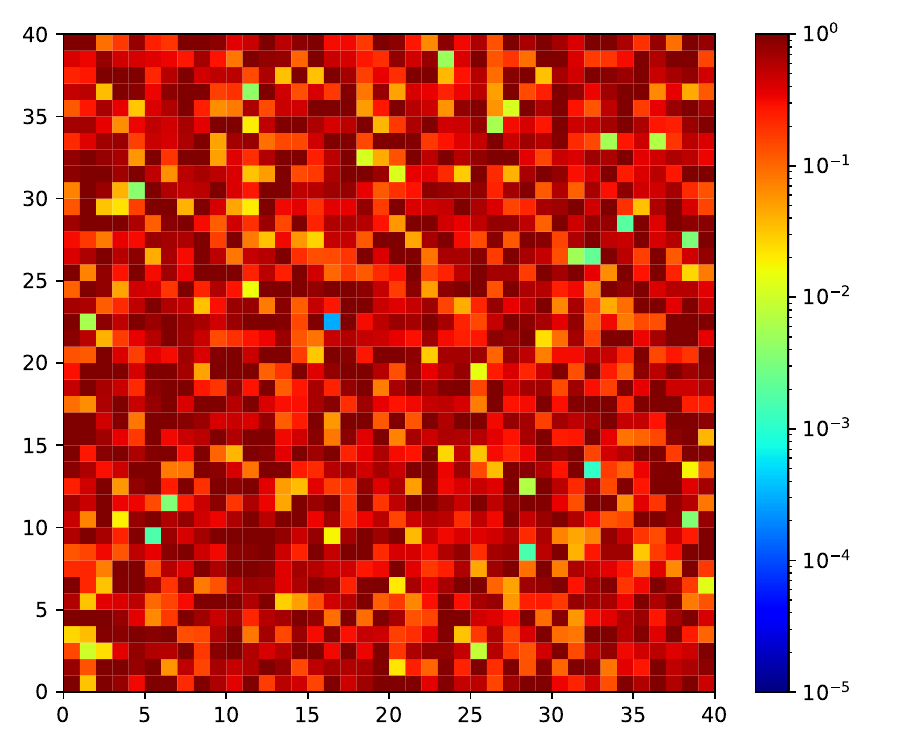}
      \caption{Point-wise errors for TT-PINNs and $It=0$.}
      \label{fig:tnn_6d_heat_error_c}
  \end{subfigure}
  \begin{subfigure}[b]{0.35\textwidth}
      \includegraphics[width=\textwidth]{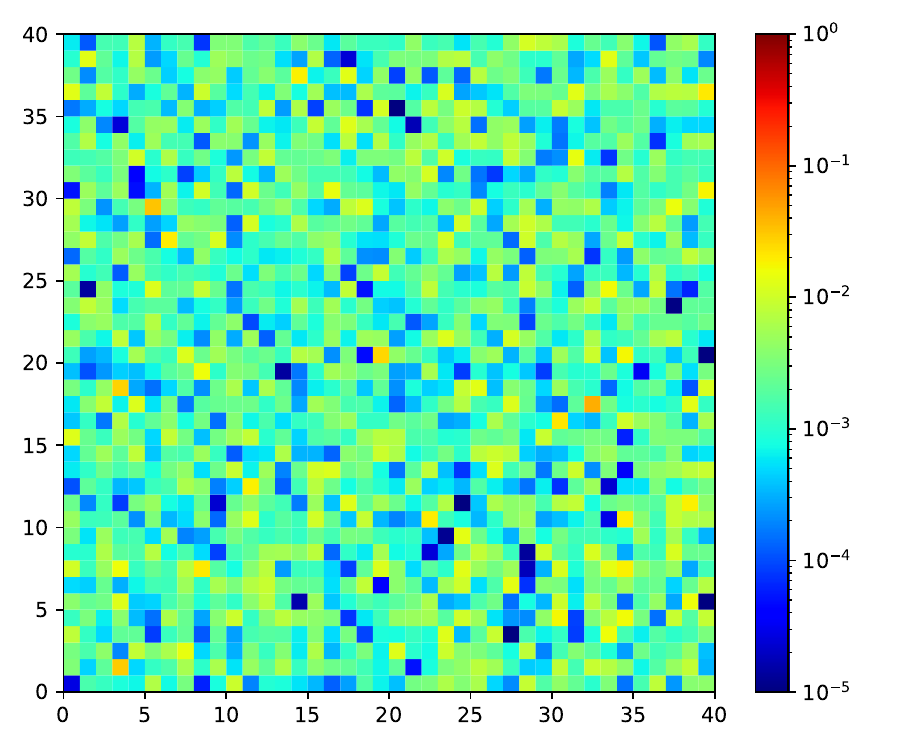}
      \caption{Point-wise errors for TT-PINNs and $It=4$.}
      \label{fig:tnn_6d_heat_error_d}
  \end{subfigure}
  \caption{\enspace Point-wise errors for heat equation (\ref{tnn_heat_eq}).}
  \label{fig:tnn_6d_heat_error}
\end{figure}

Following the approach proposed in \cite{lagaris1998artificial}, 
the initial condition is directly embedded into the neural 
network architecture, eliminating the need for a 
separate loss term. This not only ensures more accurate 
satisfaction of the initial condition but also facilitates 
more efficient extraction of frequency features. 
In this framework, the time dimension is treated in the same manner as the spatial dimensions, allowing the problem to be reformulated as a seven-dimensional PDE without explicitly distinguishing between time and space. 
All other hyperparameter settings 
remain consistent with the experimental configuration specified 
at the beginning of this section.          
We then employ the Algorithm~\ref{al} to solve the heat equation (\ref{tnn_heat_eq}), setting the parameter $M$ to 50.
To illustrate point-wise error behavior, we randomly 
select 1,600 points in the space-time domain $\Omega\times (0, T]$.
The error distributions at these points for $It=0$ and $It=4$ are
shown in Figure \ref{fig:tnn_6d_heat_error}. These results clearly demonstrate that Algorithm~\ref{al} reduces point-wise errors by one to three orders of magnitude after four adaptive iterations. 
The relative $L_2$ errors at each adaptive step are
reported in Table \ref{table_heat}. Both CP-PINNs and TT-PINNs achieve highly accurate solutions, with final relative errors of 1.605e-04 and 4.664e-03, respectively.

\begin{table}[H]
  \begin{center}
  \caption{Relative $L_2$ error at each adaptive step for heat equation (\ref{tnn_heat_eq}).}
  \begin{tabular}{cccccc}
    \hline
    Adaptive step ${It}$ & 0 & 1 & 2 & 3 & 4\\
    \hline 
    CP-PINNs & 5.432e-02 & 2.909e-04 & 1.992e-04& 1.896e-04 & 1.605e-04 \\
    \hline
    TT-PINNs & 2.450e-01 & 7.563e-03& 2.109e-03& 2.403e-03 & 4.664e-03 \\
    \hline
    \label{table_heat}
  \end{tabular}
\end{center}
\end{table}

\subsection{Wave Equation}

The third benchmark test case involves a six-dimensional wave equation given by:
\begin{equation}
  \label{tnn_wave_eq}
  \begin{split}
    &u_{tt}(\bm{x}, t) - \frac{1}{75} \Delta_{\bm{x}} u(\bm{x}, t) = 0, \quad (\bm{x}, t) \in \Omega \times (0,T], \\
    &u(\bm{x}, t) = 0, \quad (\bm{x}, t) \in \partial \Omega \times (0, T], \\
    &u(\bm{x}, 0) = \sum_{i=1}^d \sin(150\pi x_i) + \sin(300\pi x_i), \quad \bm{x} \in \Omega, \\
    &u_t(\bm{x}, 0) = 0, \quad \bm{x} \in \Omega. 
  \end{split}
\end{equation}
We take $\Omega = (0, 1)^6, \ T = 1$. 
The exact solution is given by 
\begin{equation*}
  u_{\text{exact}}(\bm{x}, t) = \sum_{i=1}^6 \sin(150\pi x_i) \cos(2\pi t) + \sin(300\pi x_i)\cos(4\pi t).
\end{equation*}
To solve the wave equation \eqref{tnn_wave_eq} numerically, we employ frequency-adaptive TNNs with $I=4$ adaptive iterations.   
As suggested in \cite{wang2024respecting}, to address the computational challenges associated with the time dimension, we adopt the following loss function:
\begin{equation}
  \label{wa_loss}
  \begin{split}
    \mathcal{L}(\bm{\theta})  =& \mathcal{L}_{u}(\bm{\theta}) + \mathcal{L}_{u_t}(\bm{\theta}) + \mathcal{L}_{r}(\bm{\theta})  =  \frac{\omega_u}{N_{u}}\sum\limits_{i=1}^{N_{u}}|u_{\text{net}}(\bm{x}^i_b, t_{u}^i;\bm{\theta})|^2 \\
    &+ 
     \frac{\omega_{u_t}}{N_{u_t}}\sum\limits_{i=1}^{N_{u_t}}\big|\frac{\partial u_{\text{net}}}{\partial t}(\bm{x}_{u_t}^i, 0;\bm{\theta})\big|^2
    +  \frac{\omega_{r}(t_r^i)}{N_r}\sum\limits_{i=1}^{N_r} \big| \frac{\partial^2 u_{\text{net}}}{\partial t^2}(\bm{x}_r^i, t_r^i;\bm{\theta}) - 
    \frac{1}{75} \Delta_{\bm{x}} u_{\text{net}}(\bm{x}_r^i, t_r^i;\bm{\theta})\big|^2,
  \end{split}
\end{equation}
where $\bm{x}^i_b\in \partial \Omega$, and the 
weights are set to $\omega_u=10,000$ and $\omega_{u_t}=1,000$. 
The initial condition is directly embedded
into the neural network, eliminating the need for an 
additional loss term. 
In (\ref{wa_loss}), the weight $\omega_{r}(t)$ 
is a continuous gate function defined as 
$\omega_{r}(t)=(1-\tanh(5({t} - \mu )))/2$, where $\mu$
is the scalar shift parameter controlling 
the fraction of time revealed to the model. 
The shift 
parameter $\mu$ is updated at each training epoch according to: 
\begin{equation*}
  \mu_{n+1} =\mu_{n} + 0.002 \mathrm{e}^{-0.005  \mathcal{L}_r(\bm{\theta})},
\end{equation*}
where $\mathcal{L}_r(\bm{\theta})$ represents the 
PDE loss at the $n$th training epoch. 
All other initialization and hyperparameter settings remain consistent with those specified at the beginning of this section.

To visualize the point-wise errors of CP-PINNs and TT-PINNs, we randomly select 1,600 sample points in the space-time domain $\Omega \times (0, T]$.
The corresponding error distributions at $It=0$ and $It=4$ are
shown in Figure \ref{fig:tnn_6d_wave_error}. 
These results clearly demonstrate that frequency adaptation
significantly reduces point-wise errors for both CP-PINNs 
and TT-PINNs, achieving reductions of at least two orders of magnitude. 
The relative $L_2$ errors at each adaptive step 
are summarized in Table \ref{table_wave}. Specifically, 
the error for CP-PINNs decreases from 
1.673e-01 to 4.797e-04, while that 
for TT-PINNs it is reduced from 
1.149 to 5.566e-03. This example illustrates that
frequency-adaptive TNNs can be effectively integrated
with classical techniques \cite{wang2024respecting} to solve the wave equation, resulting in substantial improvements in accuracy.

\begin{figure}[H]
  \centering
  \begin{subfigure}[b]{0.35\textwidth}
    \includegraphics[width=\textwidth]{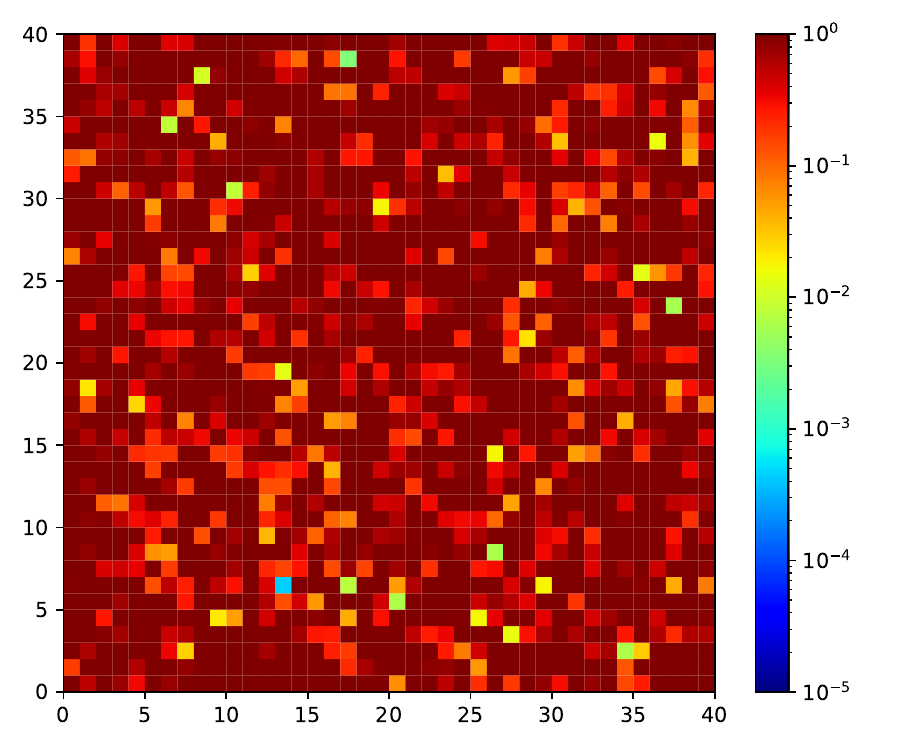}
    \caption{Point-wise errors for CP-PINNs and $It=0$.}
    \label{fig:tnn_6d_wave_error_a}
  \end{subfigure}
  \begin{subfigure}[b]{0.35\textwidth}
    \includegraphics[width=\textwidth]{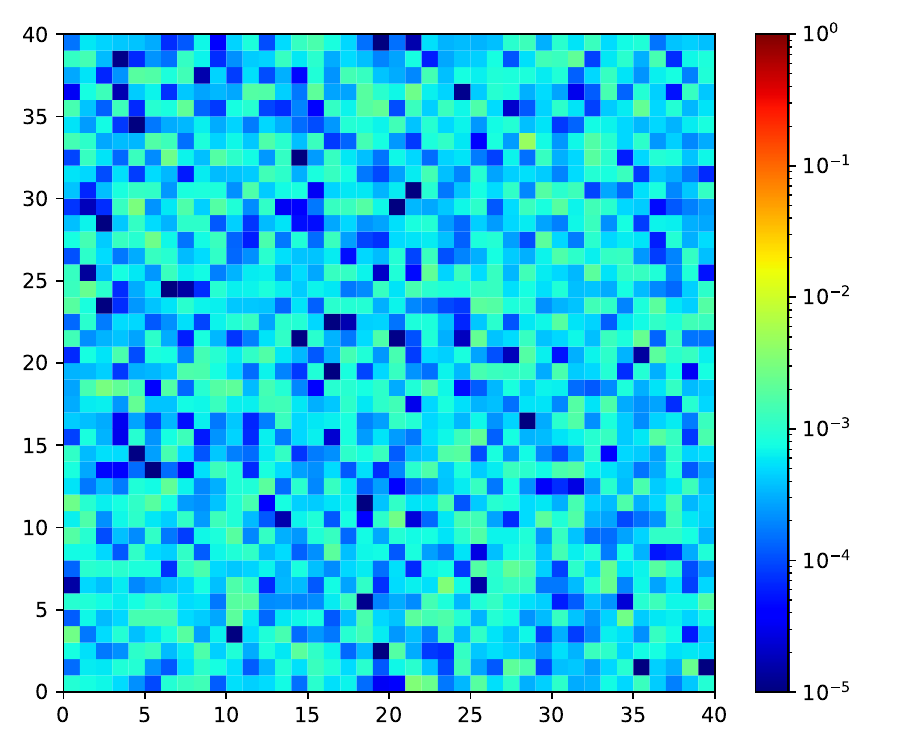}
    \caption{Point-wise errors for CP-PINNs and $It=4$.}
    \label{fig:tnn_6d_wave_error_b}
  \end{subfigure}
  \begin{subfigure}[b]{0.35\textwidth}
      \includegraphics[width=\textwidth]{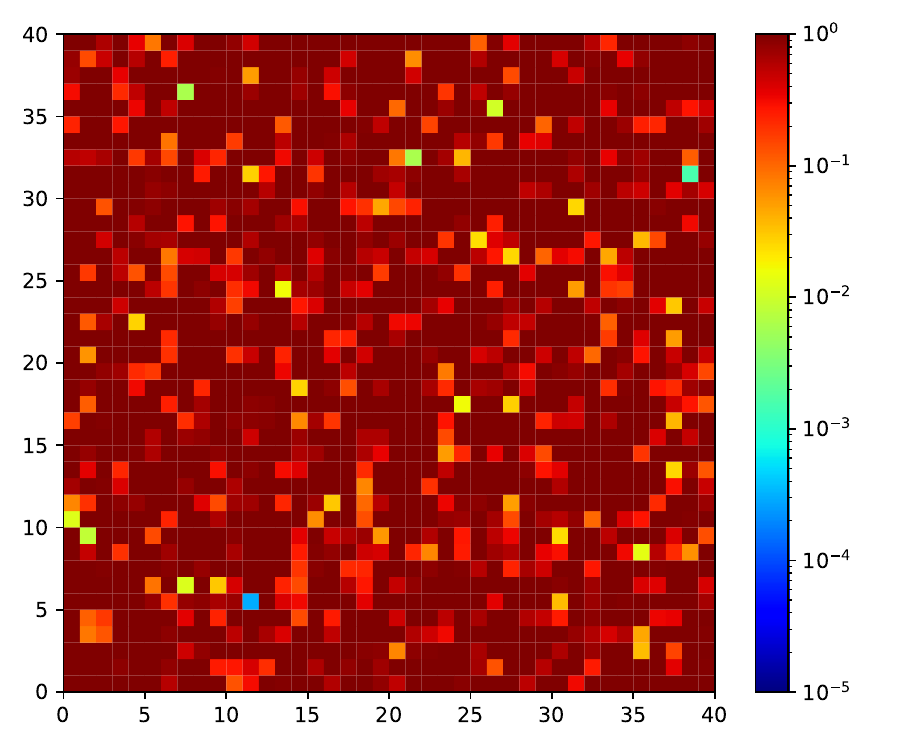}
      \caption{Point-wise errors for TT-PINNs and $It=0$.}
      \label{fig:tnn_6d_wave_error_c}
  \end{subfigure}
  \begin{subfigure}[b]{0.35\textwidth}
      \includegraphics[width=\textwidth]{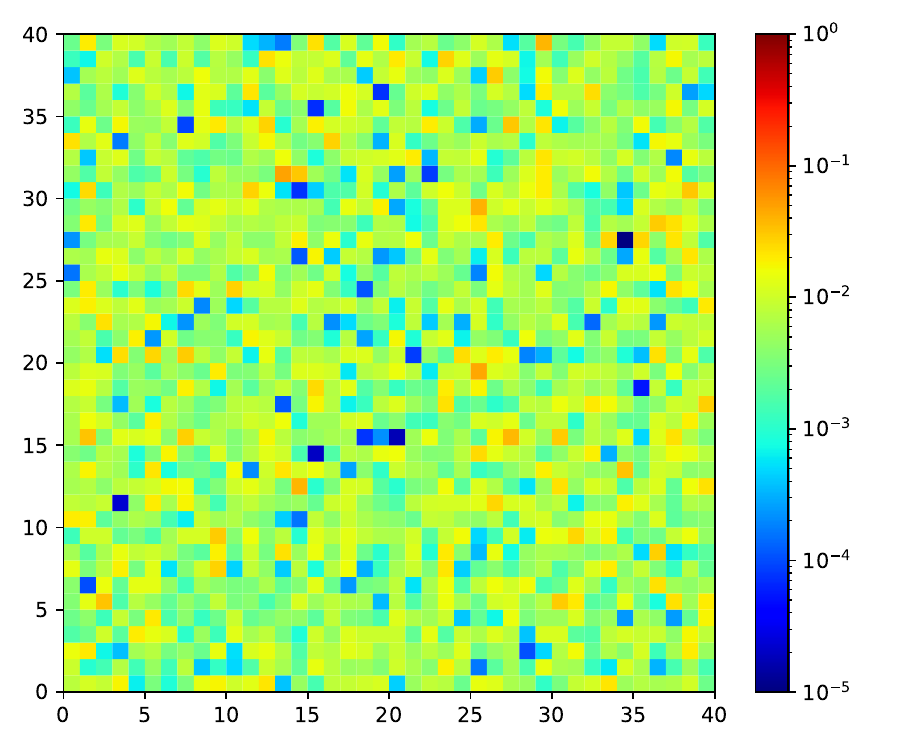}
      \caption{Point-wise errors for TT-PINNs and $It=4$.}
      \label{fig:tnn_6d_wave_error_d}
  \end{subfigure}
  \caption{\enspace Point-wise errors for wave equation (\ref{tnn_wave_eq}).}
  \label{fig:tnn_6d_wave_error}
\end{figure}

\begin{table}[H]
  \begin{center}
  \caption{Relative $L_2$ error at each adaptive step for wave equation (\ref{tnn_wave_eq}).}
  \begin{tabular}{cccccc}
    \hline
    Adaptive step ${It}$ & 0 & 1 & 2 & 3 & 4\\
    \hline 
    CP-PINNs & 1.673e-01 & 6.997e-04 & 4.760e-04& 4.590e-04 & 4.797e-04 \\
    \hline
    TT-PINNs & 1.149e+00 & 6.571e-01& 4.063e-01& 8.892e-03 & 5.566e-03 \\
    \hline
    \label{table_wave}
  \end{tabular}
\end{center}
\end{table}

\subsection{Helmholtz Equation}

Consider the following Helmholtz equation:
\begin{equation}
  \label{tnn_he_eq}
  \begin{split}
    - \Delta u(\bm{x}) + \lambda^2 u(\bm{x}) &= f(\bm{x}), \quad \bm{x}\in \Omega, \\
    u(\bm{x}) &= g(\bm{x}), \quad \bm{x} \in \partial \Omega,
  \end{split}
\end{equation}
where $\lambda=1$ and choose appropriate 
$f(\bm{x})$ and $g(\bm{x})$ such that the exact solution is:
\begin{equation*}
  u_{\text{exact}}(\bm{x}) = \sum_{i=1}^6 \sin(360 \pi x_i).
\end{equation*} 
We then employ 
the frequency-adaptive TNNs with 
$I=4$ to numerically solve the 
six-dimensional Helmholtz equation (\ref{tnn_he_eq}). 

\begin{table}[H]
  \begin{center}
  \caption{Relative $L_2$ error at each adaptive step for Helmholtz equation (\ref{tnn_he_eq}).}
  \begin{tabular}{cccccc}
    \hline
    Adaptive step ${It}$ & 0 & 1 & 2 & 3 & 4\\
    \hline 
    CP-PINNs & 2.419e-01 & 1.461e-03& 2.789e-04& 2.451e-04 & 2.402e-04 \\
    \hline
    TT-PINNs & 9.999e-01 & 1.354e-03 & 4.041e-03& 3.784e-03 & 3.545e-03 \\
    \hline
    \label{table_helmholtz}
  \end{tabular}
\end{center}
\end{table}

The corresponding relative 
$L_2$ errors for each frequency-adaptive iteration are reported 
in Table \ref{table_helmholtz}. As observed, the relative 
$L_2$ errors for CP-PINNs decrease substantially 
from an initial value of 2.419e-01 to 2.402e-04. 
Similarly, the error for TT-PINNs is reduced from 
9.999e-01 to 3.545e-03. These results demonstrate clearly 
the accuracy and robustness of the proposed frequency-adaptive 
TNN algorithm in solving high-dimensional multi-scale Helmholtz 
problems.

\section{Conclusion}
\label{conclusion}

In this study, we first introduce two TNN architectures (CP-PINNs and TT-PINNs) and 
investigate their associated Frequency Principle 
(also known as spectral bias). Using Fourier analysis, we provide a theoretical
explanation of this phenomenon in high-dimensional settings. 
To mitigate spectral bias, we incorporate Fourier features
into the TNN framework, thereby
enhancing its ability to solve high-dimensional problems with significant high-frequency 
components. Building upon the structure of TNNs, we 
further propose a novel approach to extract frequency 
features of high-dimensional functions by performing the DFT to one-dimensional component functions. This leads to the 
development of frequency-adaptive TNNs, which dynamically 
adjust embedded frequency features based on the characteristics of the problem. Compared to fixed-frequency embedding, 
the adaptive strategy significantly improves the model's 
ability to capture complex, high-frequency behaviors. 
The proposed method demonstrates strong performance across 
a wide range of problems, including the Poisson equation, 
heat, wave and Helmholtz equation.

Despite its effectiveness, the current adaptive frequency framework 
has certain limitations. A key drawback is the need to retrain the 
network from scratch at each adaptive iteration, which incurs high 
computational costs. A promising direction for future work involves designing projection-based techniques to transfer learned representations between iterations, thereby accelerating convergence. Moreover, since the DFT relies on structured and regular domains, extending the method to irregular geometries remains a challenging but important area for future research.

\section*{Acknowledgments}
\bibliography{ref}

@article{krizhevsky2012imagenet,
  title={{ImageNet} classification with deep convolutional neural networks},
  author={Krizhevsky, Alex and Sutskever, Ilya and Hinton, Geoffrey E},
  journal={Advances in Neural Information Processing Systems},
  volume={25},
  year={2012}
}

@article{hinton2012deep,
  title={Deep neural networks for acoustic modeling in speech recognition: The shared views of four research groups},
  author={Hinton, Geoffrey and Deng, Li and Yu, Dong and Dahl, George E and Mohamed, Abdel-rahman and Jaitly, Navdeep and Senior, Andrew and Vanhoucke, Vincent and Nguyen, Patrick and Sainath, Tara N and others},
  journal={IEEE Signal Processing Magazine},
  volume={29},
  number={6},
  pages={82--97},
  year={2012},
  publisher={IEEE}
}

@article{vaswani2017attention,
  title={Attention is all you need},
  author={Vaswani, Ashish and Shazeer, Noam and Parmar, Niki and Uszkoreit, Jakob and Jones, Llion and Gomez, Aidan N and Kaiser, {\L}ukasz and Polosukhin, Illia},
  journal={Advances in Neural Information Processing Systems},
  volume={30},
  year={2017}
}

@inproceedings{devlin2018bert,
  title={{BERT}: Pre-training of deep bidirectional transformers for language understanding},
  author    = "Devlin, Jacob and Chang, Ming-Wei and Lee, Kenton and Toutanova, Kristina",
  booktitle={Proceedings of the 2019 Conference of the North American Chapter of the
  Association for Computational Linguistics: Human Language Technologies,
  Volume 1 (Long and Short Papers)},
  pages={4171--4186},
  year      = "2019",
}

@article{brown2020language,
  title={Language models are few-shot learners},
  author={Brown, Tom and Mann, Benjamin and Ryder, Nick and Subbiah, Melanie and Kaplan, Jared D and Dhariwal, Prafulla and Neelakantan, Arvind and Shyam, Pranav and Sastry, Girish and Askell, Amanda and others},
  journal={Advances in Neural Information Processing Systems},
  volume={33},
  pages={1877--1901},
  year={2020}
}

@article{han2017deep,
  title={Deep learning-based numerical methods for high-dimensional parabolic partial differential equations and backward stochastic differential equations},
  author={Han, Jiequn and Jentzen, Arnulf and others},
  journal={Communications in Mathematics and Statistics},
  volume={5},
  number={4},
  pages={349--380},
  year={2017},
  publisher={Springer}
}

@article{yu2018deep,
  title={The deep {Ritz} method: A deep learning-based numerical algorithm for solving variational problems},
  author={Yu, Bing and others},
  journal={Communications in Mathematics and Statistics},
  volume={6},
  number={1},
  pages={1--12},
  year={2018},
  publisher={Springer}
}

@article{zang2020weak,
  title={Weak adversarial networks for high-dimensional partial differential equations},
  author={Zang, Yaohua and Bao, Gang and Ye, Xiaojing and Zhou, Haomin},
  journal={Journal of Computational Physics},
  volume={411},
  pages={109409},
  year={2020},
  publisher={Elsevier}
}

@article{tran2019dnn,
  title={{DNN} approximation of nonlinear finite element equations},
  author={Tran, Tuyen and Hamilton, Aidan and McKay, Maricela Best and Quiring, Benjamin and Vassilevski, Panayot S},
  journal={arXiv preprint arXiv:1911.05240},
  year={2019}
}

@article{han2018solving,
  title={Solving high-dimensional partial differential equations using deep learning},
  author={Han, Jiequn and Jentzen, Arnulf and E, Weinan},
  journal={Proceedings of the National Academy of Sciences},
  volume={115},
  number={34},
  pages={8505--8510},
  year={2018},
  publisher={National Acad Sciences}
}

@article{han2017deep2,
  title={Deep Potential: A general Representation of a Many-Body Potential Energy Surface},
  author={Han, Jiequn and Zhang, Linfeng and Car, Roberto and others},
  journal={Communications in Computational Physics,},
  volume={23},
  number={3},
  pages={629--639},
  year={2018},
  publisher={Global Science Press}
}

@article{he2018relu,
  title={ReLU deep neural networks and linear finite elements},
  author={He, Juncai and Li, Lin and Xu, Jinchao and Zheng, Chunyue},
  journal={Journal of Computational Mathematics,},
  volume={38},
  number={3},
  pages={502--527},
  year={2020}
}

@article{liao2019deep,
  title={Deep {Nitsche} Method: Deep {Ritz} Method with Essential Boundary Conditions},
  author={Ming, Yulei Liao and others},
  journal={Communications in Computational Physics,},
  volume={29},
  number={5},
  pages={1365--1384},
  year={2021}
}

@article{raissi2019physics,
  title={Physics-informed neural networks: A deep learning framework for solving forward and inverse problems involving nonlinear partial differential equations},
  author={Raissi, Maziar and Perdikaris, Paris and Karniadakis, George E},
  journal={Journal of Computational Physics},
  volume={378},
  pages={686--707},
  year={2019},
  publisher={Elsevier}
}

@article{strofer2019data,
  title={Data-driven, physics-based feature extraction from fluid flow fields using convolutional neural networks},
  author={Strofer, Carlos Michelen and Wu, Jin-Long and Xiao, Heng and Paterson, Eric},
  journal={Communications in Computational Physics},
  volume={25},
  number={3},
  pages={625--650},
  year={2019},
  publisher={GLOBAL SCIENCE PRESS Office B, 9/F, Kings Wing Plaza2, No. 1 On Kwan St~…}
}

@article{wang2020mesh,
  title={A mesh-free method for interface problems using the deep learning approach},
  author={Wang, Zhongjian and Zhang, Zhiwen},
  journal={Journal of Computational Physics},
  volume={400},
  pages={108963},
  year={2020},
  publisher={Elsevier}
}

@inproceedings{rahaman2019spectral,
  title={On the spectral bias of neural networks},
  author={Rahaman, Nasim and Baratin, Aristide and Arpit, Devansh and Draxler, Felix and Lin, Min and Hamprecht, Fred and Bengio, Yoshua and Courville, Aaron},
  booktitle={International Conference on Machine Learning},
  pages={5301--5310},
  year={2019},
  organization={PMLR}
}

@article{xu2018understanding,
  title={Understanding training and generalization in deep learning by {Fourier} analysis},
  author={Xu, Zhiqin John},
  journal={arXiv preprint arXiv:1808.04295},
  year={2018}
}

@article{xu2019frequency,
  title={Frequency Principle: {Fourier} Analysis Sheds Light on Deep Neural Networks},
  author={Xu, Zhi-Qin John},
  journal={Communications in Computational Physics},
  volume={28},
  number={5},
  pages={1746--1767},
  year={2020}
}

@article{zhang2019explicitizing,
  title={Explicitizing an implicit bias of the frequency principle in two-layer neural networks},
  author={Zhang, Yaoyu and Xu, Zhi-Qin John and Luo, Tao and Ma, Zheng},
  journal={arXiv preprint arXiv:1905.10264},
  year={2019}
}

@inproceedings{xu2019training,
  title={Training behavior of deep neural network in frequency domain},
  author={Xu, Zhi-Qin John and Zhang, Yaoyu and Xiao, Yanyang},
  booktitle={International Conference on Neural Information Processing},
  pages={264--274},
  year={2019},
  organization={Springer}
}

@article{dahmen2016tensor,
  title={Tensor-sparsity of solutions to high-dimensional elliptic partial differential equations},
  author={Dahmen, Wolfgang and DeVore, Ronald and Grasedyck, Lars and S{\"u}li, Endre},
  journal={Foundations of Computational Mathematics},
  volume={16},
  number={4},
  pages={813--874},
  year={2016},
  publisher={Springer}
}

@article{schwab2011sparse,
  title={Sparse tensor discretizations of high-dimensional parametric and stochastic {PDEs}},
  author={Schwab, Christoph and Gittelson, Claude Jeffrey},
  journal={Acta Numerica},
  volume={20},
  pages={291--467},
  year={2011},
  publisher={Cambridge University Press}
}

@article{zeng2022adaptive,
  title={Adaptive deep neural networks methods for high-dimensional partial differential equations},
  author={Zeng, Shaojie and Zhang, Zong and Zou, Qingsong},
  journal={Journal of Computational Physics},
  volume={463},
  pages={111232},
  year={2022},
  publisher={Elsevier}
}

@article{cohen2011analytic,
  title={Analytic regularity and polynomial approximation of parametric and stochastic elliptic {PDE's}},
  author={Cohen, Albert and Devore, Ronald and Schwab, Christoph},
  journal={Analysis and Applications},
  volume={9},
  number={01},
  pages={11--47},
  year={2011},
  publisher={World Scientific}
}

@article{bachmayr2017kolmogorov,
  title={Kolmogorov widths and low-rank approximations of parametric elliptic {PDEs}},
  author={Bachmayr, Markus and Cohen, Albert},
  journal={Mathematics of Computation},
  volume={86},
  number={304},
  pages={701--724},
  year={2017}
}

@article{novikov2015tensorizing,
  title={Tensorizing neural networks},
  author={Novikov, Alexander and Podoprikhin, Dmitrii and Osokin, Anton and Vetrov, Dmitry P},
  journal={Advances in Neural Information Processing Systems},
  pages={442--450},
  year={2015}
}

@article{jin2022mionet,
  title={{MIONet}: Learning multiple-input operators via tensor product},
  author={Jin, Pengzhan and Meng, Shuai and Lu, Lu},
  journal={SIAM Journal on Scientific Computing},
  volume={44},
  number={6},
  pages={A3490--A3514},
  year={2022},
  publisher={SIAM}
}

@inproceedings{vemuri2025functional,
  title={Functional tensor decompositions for physics-informed neural networks},
  author={Vemuri, Sai Karthikeya and B{\"u}chner, Tim and Niebling, Julia and Denzler, Joachim},
  booktitle={International Conference on Pattern Recognition},
  pages={32--46},
  year={2025},
  organization={Springer}
}

@article{wang2022tensor,
  title={Tensor neural network and its numerical integration},
  author  = {Wang, Yifan and Xie, Hehu and Jin, Pengzhan},
  journal = {Journal of Computational Mathematics},
  volume  = {42},
  number  = {6},
  pages   = {1714--1742},
  year    = {2024},
}

@article{liu2020multi,
  title={Multi-scale deep neural network {(MscaleDNN)} for solving {Poisson-Boltzmann} equation in complex domains},
  author={Liu, Ziqi and Cai, Wei and Xu, Zhi-Qin John},
  journal = {Communications in Computational Physics},
  volume = {28},
  number = {5},
  pages = {1970--2001},
  year = {2020},
}

@article{zhang2023correction,
  title={A correction and comments on "multi-scale deep neural network {(MscaleDNN)} for solving {Poisson-Boltzmann} equation in complex domains {CICP}, 28 (5): 1970--2001, 2020"},
  author={Zhang, Lulu and Cai, Wei and Xu, Zhi-Qin John},
  journal={Communications in Computational Physics},
  volume={33},
  number={5},
  pages={1509--1513},
  year={2023}
}

@article{tancik2020fourier,
  title={Fourier features let networks learn high frequency functions in low dimensional domains},
  author={Tancik, Matthew and Srinivasan, Pratul and Mildenhall, Ben and Fridovich-Keil, Sara and Raghavan, Nithin and Singhal, Utkarsh and Ramamoorthi, Ravi and Barron, Jonathan and Ng, Ren},
  journal={Advances in Neural Information Processing Systems},
  volume={33},
  pages={7537--7547},
  year={2020}
}

@article{wang2021eigenvector,
  title={On the eigenvector bias of {Fourier} feature networks: From regression to solving multi-scale {PDEs} with physics-informed neural networks},
  author={Wang, Sifan and Wang, Hanwen and Perdikaris, Paris},
  journal={Computer Methods in Applied Mechanics and Engineering},
  volume={384},
  pages={113938},
  year={2021},
  publisher={Elsevier}
}

@article{huang2024frequency,
    title={Frequency-adaptive Multi-scale Deep Neural Networks},
    author={Huang, Jizu and You, Rukang and Zhou, Tao},
    journal = {Computer Methods in Applied Mechanics and Engineering},
    volume = {437},
    pages = {117751},
    year = {2025},
}

@article{falco2019dirac,
  title={On the {Dirac--Frenkel} variational principle on tensor {Banach} spaces},
  author={Falc{\'o}, Antonio and Hackbusch, Wolfgang and Nouy, Anthony},
  journal={Foundations of Computational Mathematics},
  volume={19},
  pages={159--204},
  year={2019},
  publisher={Springer}
}

@article{patera2007reduced,
  title={Reduced basis approximation and a posteriori error estimation for parametrized partial differential equations},
  author  = {Rozza, Gianluigi and Huynh, Duc B. P. and Patera, Anthony T.},
  journal = {Archives of Computational Methods in Engineering},
  volume  = {15},
  number  = {3},
  pages   = {229--275},
  year    = {2008},
}

@article{khoromskij2011tensor,
  title={Tensor-structured {Galerkin} approximation of parametric and stochastic elliptic {PDEs}},
  author={Khoromskij, Boris N and Schwab, Christoph},
  journal={SIAM Journal on Scientific Computing},
  volume={33},
  number={1},
  pages={364--385},
  year={2011},
  publisher={SIAM}
}

@article{chen2024solving,
  title={Solving High-dimensional Parametric Elliptic Equation Using Tensor Neural Network},
  author={Chen, Hongtao and Fu, Rui and Wang, Yifan and Xie, Hehu},
  journal={Journal of Intelligent Algorithms and Scientific Computing},
  volume={1},
  number={1},
  pages={88--110},
  year    = {2026},
}

@article{hitchcock1927expression,
  title={The expression of a tensor or a polyadic as a sum of products},
  author={Hitchcock, Frank L},
  journal={Journal of Mathematics and Physics},
  volume={6},
  number={1-4},
  pages={164--189},
  year={1927},
  publisher={Wiley Online Library}
}

@article{tucker1966some,
  title={Some mathematical notes on three-mode factor analysis},
  author={Tucker, Ledyard R},
  journal={Psychometrika},
  volume={31},
  number={3},
  pages={279--311},
  year={1966},
  publisher={Springer}
}

@article{oseledets2011tensor,
  title={Tensor-train decomposition},
  author={Oseledets, Ivan V},
  journal={SIAM Journal on Scientific Computing},
  volume={33},
  number={5},
  pages={2295--2317},
  year={2011},
  publisher={SIAM}
}

@article{wang2024solving,
  title={Solving high-dimensional partial differential equations using tensor neural network and a posteriori error estimators},
  author={Wang, Yifan and Lin, Zhongshuo and Liao, Yangfei and Liu, Haochen and Xie, Hehu},
  journal={Journal of Scientific Computing},
  volume={101},
  number={3},
  pages={1--29},
  year={2024},
  publisher={Springer}
}

@article{rahimi2007random,
  title={Random features for large-scale kernel machines},
  author={Rahimi, Ali and Recht, Benjamin},
  journal={Advances in Neural Information Processing Systems},
  volume={20},
  year={2007}
}

@inproceedings{zhong2019reconstructing,
  title={Reconstructing continuous distributions of {3D} protein structure from {cryo-EM} images},
  author={Zhong, Ellen D and Bepler, Tristan and Davis, Joseph H and Berger, Bonnie},
  booktitle = {International Conference on Learning Representations},
  year      = {2020},
}

@article{chen2023adaptive,
  title={Adaptive multi-scale neural network with {Resnet} blocks for solving partial differential equations},
  author={Chen, Miaomiao and Niu, Ruiping and Zheng, Wen},
  journal={Nonlinear Dynamics},
  volume={111},
  number={7},
  pages={6499--6518},
  year={2023},
  publisher={Springer}
}

@article{cai2019multi,
  title={Multi-scale deep neural networks for solving high dimensional {PDEs}},
  author={Cai, Wei and Xu, Zhi-Qin John},
  journal={arXiv preprint arXiv:1910.11710},
  year={2019}
}

@inproceedings{glorot2010understanding,
  title={Understanding the difficulty of training deep feedforward neural networks},
  author={Glorot, Xavier and Bengio, Yoshua},
  booktitle={Proceedings of the Thirteenth International Conference on Artificial Intelligence and Statistics},
  pages={249--256},
  year={2010},
  organization={JMLR Workshop and Conference Proceedings}
}

@inproceedings{kingma2014adam,
  title     = {Adam: A Method for Stochastic Optimization},
  author    = {Kingma, Diederik P. and Ba, Jimmy},
  booktitle = {International Conference on Learning Representations},
  year      = {2015},
}

@article{lagaris1998artificial,
  title={Artificial neural networks for solving ordinary and partial differential equations},
  author={Lagaris, Isaac E and Likas, Aristidis and Fotiadis, Dimitrios I},
  journal={IEEE Transactions on Neural Networks},
  volume={9},
  number={5},
  pages={987--1000},
  year={1998},
  publisher={IEEE}
}

@article{wang2024respecting,
  title={Respecting causality for training physics-informed neural networks},
  author={Wang, Sifan and Sankaran, Shyam and Perdikaris, Paris},
  journal={Computer Methods in Applied Mechanics and Engineering},
  volume={421},
  pages={116813},
  year={2024},
  publisher={Elsevier}
}

\appendix
\section{Computational Details}
\label{cd}
The gradient of spectral loss $L(k_x, k_y)$ with respect to $w_{x,j}$ is calculated as:
\begin{equation*}
  \begin{split}
  \frac{\partial L(k_x, k_y)}{\partial w_{x,j}} &= \overline{D(k_x, k_y)} \frac{\partial D(k_x, k_y)}{\partial w_{x,j}} + D(k_x, k_y)\frac{\partial \overline{D(k_x, k_y)}}{\partial w_{x,j}}\\
  &=|D(k_x, k_y)|\frac{a_j}{k_xk_y}\left(\left(\mathrm{e}^{-\mathrm{i}\theta(k_x,\, k_y)}\mathrm{e}^{\mathrm{i}\theta_j}+\mathrm{e}^{\mathrm{i}\theta(k_x,\, k_y)}\mathrm{e}^{-\mathrm{i}\theta_j}\right)\frac{\partial D_j}{\partial w_{x,j}}
  +\mathrm{i}\left(\mathrm{e}^{\mathrm{i}\theta(k_x,\, k_y)}\mathrm{e}^{-\mathrm{i}\theta_j}-\mathrm{e}^{-\mathrm{i}\theta(k_x,\, k_y)}\mathrm{e}^{\mathrm{i}\theta_j}\right)D_j\frac{k_x b_{x,j}}{w^2_{x,j}}\right)\\
  &=2|D(k_x, k_y)|\frac{a_j}{k_xk_y}\left(
  \cos(\theta_j-\theta(k_x,\, k_y))\frac{\partial D_j}{\partial w_{x,j}}
  +\sin(\theta_j-\theta(k_x,\, k_y))
  D_j\frac{k_x b_{x,j}}{w^2_{x,j}}\right)\\
  &=2|D(k_x, k_y)|D_j\frac{a_j}{k_xk_y}\frac{ 1}{w^2_{x,j}}\Big(
  \cos(\theta_j-\theta(k_x,\, k_y))\alpha_1
  +\sin(\theta_j-\theta(k_x,\, k_y))k_xb_{x,j}
  \Big)\\
  &=2|D(k_x, k_y)|D_j\frac{a_j}{k_xk_y}\frac{ 1}{w^2_{x,j}} \sqrt{\alpha_1^2+k_x^2b_{x,j}^2} \cos(\theta_j-\theta(k_x,\, k_y)-\theta_0),
    \end{split}
\end{equation*}
where 
$\alpha_1=\frac{ \left[w_{x,j}\sinh(\pi z_x/ 2)-\frac{\pi k_x}{2}\cosh(\pi z_x / 2)\right]}{\sinh(\pi z_x / 2)}
$ is bounded as $w_{x,j}$ tends to 0 and $\cos(\theta_0)=\alpha_1/\sqrt{\alpha_1^2+k_x^2b_{x,j}^2} $. 

The gradient of spectral loss $L(k_x, k_y)$ with respect to $b_{x,j}$ is calculated as:
\begin{equation*}
  \begin{split}
  \frac{\partial L(k_x, k_y)}{\partial b_{x,j}} &= \overline{D(k_x, k_y)} \frac{\partial D(k_x, k_y)}{\partial b_{x,j}} + D(k_x, k_y)\frac{\partial \overline{D(k_x, k_y)}}{\partial b_{x,j}}\\
  &=\mathrm{i}|D(k_x, k_y)|\frac{a_j}{w_{x,j}k_y}D_j\left(\mathrm{e}^{-\mathrm{i}\theta(k_x,\, k_y)}\mathrm{e}^{\mathrm{i}\theta_j}-\mathrm{e}^{\mathrm{i}\theta(k_x,\, k_y)}\mathrm{e}^{-\mathrm{i}\theta_j}\right)\\
  &=2|D(k_x, k_y)|\frac{a_j}{w_{x,j}k_y}D_j\sin(\theta_j-\theta(k_x,\, k_y))\\&
  =2|D(k_x, k_y)|\frac{a_j}{w_{x,j}k_y}D_j\cos(\theta_j-\theta(k_x,\, k_y)-\frac\pi2).
    \end{split}
\end{equation*}
The gradients of spectral loss $L(k_x, k_y)$ with respect to $w_{y,j}$ and $b_{y,j}$ can be obtained similarly.


\end{document}